%% file: PAMI16_StructMatFact.tex
\pdfoutput=1

\documentclass[10pt,journal,compsoc]{IEEEtran}

\usepackage{cite}
\usepackage{amsmath,amsfonts,amssymb}
\usepackage{graphicx}
\usepackage{array}
\usepackage{algorithm}
\usepackage{algorithmic}
\usepackage{color}
	
\def\Tr{\text{Tr}}

\def\Nplus{\mathbb{N}_+}
\def\card{\textnormal{card}}
\DeclareMathOperator*{\argmin}{arg\,min}
\def\Re{\mathbb{R}}

\newcommand{\ts}{\textsuperscript}

\def\rank{\text{rank}}
\def\prox{\text{\bf{prox}}}
\def \st{\ \ \textnormal{s.t.} \ \ }
\def \ST{\ \ \textnormal{s.t.} \ \ }

\def\citep{\cite}
\def\citet{\cite}

\newtheorem{theorem}{Theorem}

\newtheorem{corollary}{Corollary}

\newtheorem{definition}{Definition}

\newtheorem{proposition}{Proposition}

\newenvironment{proof}[1][Proof]{\textbf{#1.} }{\ \rule{0.5em}{0.5em}}

\begin{document}

\title{Structured Low-Rank Matrix Factorization: Optimality, Algorithm, and Applications to Image Processing}
\title{Theory and Applications of Structured Matrix Factorization}
\title{Global Optimality in Structured Matrix Factorization}
\title{Structured Low-Rank Matrix Factorization: Global Optimality, Algorithms, and Applications}

\author{Benjamin D. Haeffele,~\IEEEmembership{}
        Ren\'{e} Vidal,~\IEEEmembership{Fellow,~IEEE}
\IEEEcompsocitemizethanks{\IEEEcompsocthanksitem The authors are with the Center for Imaging Science, Department of Bio\-medical Engineering, Johns Hopkins University, Baltimore,
MD, 21218.
}
\thanks{Manuscript received June 30, 2017.}}

\markboth{Journal of \LaTeX\ Class Files,~Vol.~14, No.~8, August~2015}%
{Shell \MakeLowercase{\textit{et al.}}: Bare Advanced Demo of IEEEtran.cls for IEEE Computer Society Journals}

\IEEEtitleabstractindextext{
\begin{abstract}
Recently, convex formulations of low-rank matrix factorization problems have received considerable attention in machine learning. However, such formulations often require solving for a matrix of the size of the data matrix, making it challenging to apply them to large scale datasets. Moreover, in many applications the data can display structures beyond simply being low-rank, e.g., images and videos present complex spatio-temporal structures that are largely ignored by standard low-rank methods. In this paper we study a matrix factorization technique that is suitable for large datasets and captures additional structure in the factors by using a particular form of regularization that includes well-known regularizers such as total variation and the nuclear norm as particular cases. Although the resulting optimization problem is non-convex, we show that if the size of the factors is large enough, under certain conditions, any local minimizer for the factors yields a global minimizer. A few practical algorithms are also provided to solve the matrix factorization problem, and bounds on the distance from a given approximate solution of the optimization problem to the global optimum are derived. Examples in neural calcium imaging video segmentation and hyperspectral compressed recovery show the advantages of our approach on high-dimensional datasets.
\end{abstract}

}

\maketitle

\IEEEdisplaynontitleabstractindextext
\IEEEpeerreviewmaketitle

\sloppy

\input{PAMI16_StructMatFact_introduction}

\input{PAMI16_StructMatFact_background}

\input{PAMI16_StructMatFact_factorization}

\input{PAMI16_StructMatFact_applications}

\input{PAMI16_StructMatFact_conclusion}

\bibliographystyle{IEEEtran}
\bibliography{../biblio/sparse,../biblio/vidal,../biblio/learning,../biblio/biomedical}

\vspace{-6mm}
\begin{IEEEbiography}
[{\includegraphics[width=1in,height=1.25in,clip,keepaspectratio]{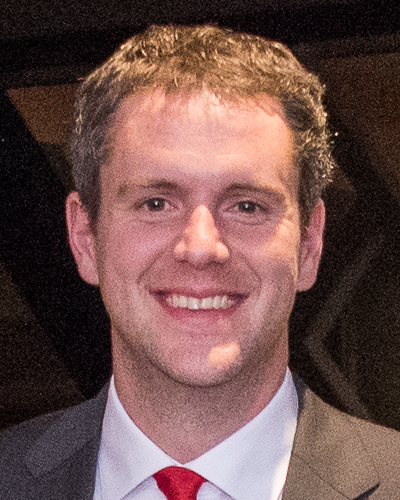}}]{Benjamin Haeffele}
is an Associate Research Scientist in the Center for Imaging Science at Johns Hopkins University.  His research interests include multiple topics in representation learning, matrix/tensor factorization, sparse and low-rank methods, optimization, phase recovery, subspace clustering, and applications of machine learning in medicine, neuroscience, and microscopy.  He received his Ph.D. in Biomedical Engineering at Johns Hopkins University in 2015 and his B.S. in Electrical Engineering from the Georgia Institute of Technology in 2006.  
\end{IEEEbiography}

\vspace{-5mm}
\begin{IEEEbiography}
[{\includegraphics[width=1in,height=1.25in,clip,keepaspectratio]{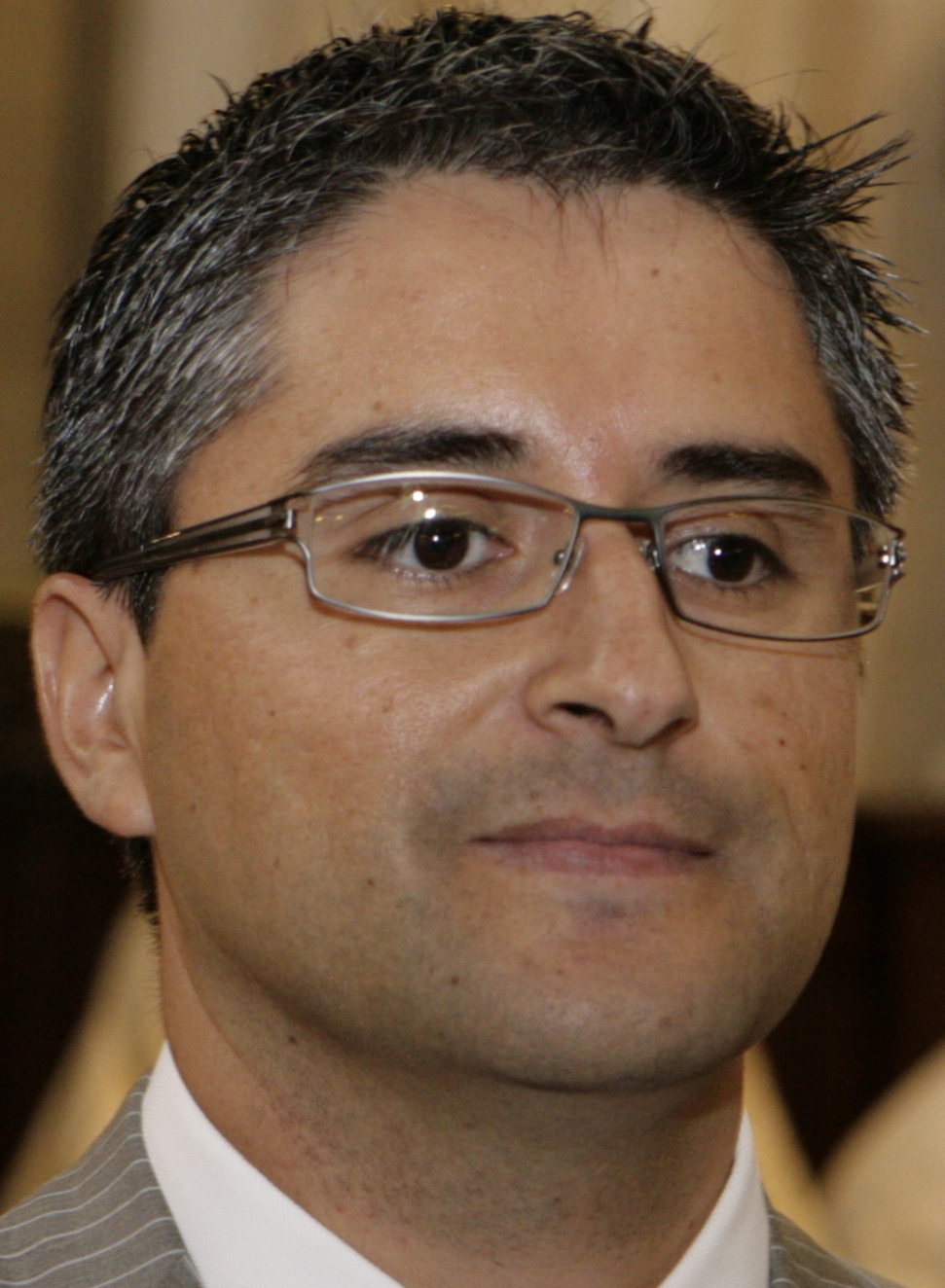}}]{Ren\'e Vidal} (F13) received the B.S. degree in electrical engineering from the Pontificia Universidad Catolica de Chile in 1997, and the M.S. and Ph.D. degrees in electrical engineering and computer sciences from the University of California at Berkeley, Berkeley, in 2000 and 2003, respectively. He has been a faculty member in the Department of Biomedical Engineering, Johns Hopkins University since 2004. He co-authored the book \textit{Generalized Principal Component Analysis} and has co-authored over 200 articles in biomedical image analysis, computer vision, machine learning, hybrid systems, and robotics. He is an Associate Editor of \textit{IEEE Transactions on Pattern Analysis and Machine Intelligence}, the \textit{SIAM Journal on Imaging Sciences}, and \textit{Computer Vision and Image Understanding}. He is a member of the ACM and the SIAM. 
\end{IEEEbiography}

\newpage
\appendices
\input{PAMI16_StructMatFact_supplement}

\end{document}

%% file: PAMI16_StructMatFact_introduction.tex
\ifCLASSOPTIONcompsoc
\IEEEraisesectionheading{\section{Introduction}\label{sec:introduction}}
\else
\section{Introduction}
\label{sec:introduction}
\fi
\vspace{-0.2mm}

In many large datasets, relevant information often lies in a subspace of much lower dimension than the ambient space, and thus the goal of many learning algorithms can be broadly interpreted as trying to find or exploit this underlying ``structure" that is present in the data.  One structure that is particularly useful both due to its wide-ranging applicability and efficient computation is the linear subspace model.  Generally speaking, if one is given $N$ data points from a $D$ dimensional ambient space, $Y = [Y_1, \ Y_2, \ \ldots, Y_N] \in \Re^{D \times N}$, a linear subspace model simply implies that there exists matrices $(U,V)$ such that $Y \approx UV^T$.  When one of the factors is known \textit{a priori}, the problem of finding the other factor simplifies considerably, but if both factors are allowed to be arbitrary one can always find an infinite number of $(U,V)$ matrices that yield the same product.  As a result, to accomplish anything meaningful, one must impose some restrictions on the factors. This idea leads to a variety of common matrix factorization techniques.  A few well known examples are the following:
\begin{itemize}
\item \textbf{Principal Component Analysis (PCA):} The number of columns, $r$, in $(U,V)$ is typically constrained to be small, e.g., $r \ll \min \{D,N\}$, and $U$ is constrained to have orthonormal columns, i.e., $U^TU = I$.
\item \textbf{Nonnegative Matrix Factorization (NMF):} The number of columns in $(U,V)$ is constrained to be small, and $(U,V)$ are required to be non-negative \cite{Lee:NIPS2001,Lee:Nature1999}.
\item \textbf{Sparse Dictionary Learning (SDL):} The number of columns in $(U,V)$ is allowed to be larger than $D$ or $N$, the columns of $U$ are required to have unit $l_2$ norm, and $V$ is required to be sparse as measured by, e.g., the $l_1$ norm or the $l_0$ pseudo-norm \cite{Elad:TSP06,Mairal:JMLR2010}.\footnotemark[1]
\end{itemize}
\footnotetext[1]{As a result, in SDL, one does not assume that there exists a single low-dimensional subspace to model the data, but rather that the data lie in a union of a large number of low-dimensional subspaces.}
\setcounter{footnote}{1}

Mathematically, the general problem of recovering structured linear subspaces from a dataset can be captured by a \textit{\textbf{structured matrix factorization problem}} of the form
\begin{equation}
\label{fact_reg_loss}
\min_{U,V} \ell(Y,UV^T) + \lambda \Theta(U,V),
\end{equation}
where $\ell$ is some \textit{\textbf{loss function}} that measures how well $Y$ is approximated by $UV^T$ and $\Theta$ is a \textit{\textbf{regularizer}} that encourages or enforces specific properties in $(U,V)$.  By taking an appropriate combination of $\ell$ and $\Theta$ one can formulate both unsupervised learning techniques, such as PCA, NMF, and SDL, or supervised learning techniques like discriminative dictionary learning \cite{Jiang:CVPR2011,Mairal:NIPS2009} and learning max-margin factorized classifiers \cite{Srebro:NIPS2004}.  However, while there are wide-ranging applications for structured matrix factorization methods that have achieved good empirical success, the associated optimization problem \eqref{fact_reg_loss} is non-convex regardless of the choice of $\ell$ and $\Theta$ functions due to the presence of the matrix product $UV^T$.  As a result, aside from a few special cases (such as PCA), finding solutions to \eqref{fact_reg_loss} poses a significant challenge, which often requires one to instead consider approximate solutions that depend on a particular choice of initialization and optimization method.

Given the challenge of non-convex optimization, one possible approach to matrix factorization is to relax the non-convex problem into a problem which is convex on the product of the factorized matrices, $X=UV^T$, and then recover the factors of $X$ after solving the convex relaxation.  As a concrete example, in low-rank matrix factorization, one might be interested in solving a problem of the form
\begin{equation}
\label{eq:low_rank_X}
\min_X \ell(Y,X) \st \rank(X) \leq r,
\end{equation}
which is equivalently defined as a factorization problem
\begin{equation}
\label{eq:low_rank_fact}
\min_{U,V} \ell(Y,UV^T),
\end{equation}
where the rank constraint is enforced by limiting the number of columns in $(U,V)$  to be less than or equal to $r$.  However, aside from a few special choices of $\ell$, solving \eqref{eq:low_rank_X} or \eqref{eq:low_rank_fact} is in general an NP-hard problem.  Instead, one can relax \eqref{eq:low_rank_X} into a convex problem by using a convex regularizer that promotes low-rank solutions, such as the nuclear norm $\|X\|_*$ (sum of the singular values of $X$), and then solve
\begin{equation}
\label{eq:priorwork_nuclear_norm}
\min_X \ell(Y,X) + \lambda \|X\|_*,
\end{equation}
which can often be done efficiently if $\ell(Y,X)$ is convex with respect to $X$ \cite{Cai:SJO08,Recht:SIAM10}.  Given a solution to \eqref{eq:priorwork_nuclear_norm}, $\hat{X}$, it is then simple to find a low-rank factorization $UV^T = \hat{X}$ via a singular value decomposition.  Unfortunately,  while the nuclear norm provides a nice convex relaxation for low-rank matrix factorization problems, nuclear norm relaxation does not capture the full generality of problems such as \eqref{fact_reg_loss} as it does not necessarily ensure that $\hat{X}$ can be factorized as $\hat{X}=UV^T$ for some $(U,V)$ pair which has the desired structure encouraged by $\Theta(U,V)$ (e.g., in non-negative matrix factorization we require $U$ and $V$ to be non-negative), nor does it provide a means to find the desired factors.

Based on the above discussion, optimization problems in the factorized space, such as \eqref{fact_reg_loss}, versus problems in the product space, with \eqref{eq:priorwork_nuclear_norm} as a particular example, both present various advantages and disadvantages.  Factorized problems attempt to solve for the desired factors $(U,V)$ directly, provide significantly increased modeling flexibility by permitting one to model structure on the factors (sparsity, non-negativity, etc.), and allow one to potentially work with a significantly reduced number of variables if the number of columns in $(U,V)$ is $\ll \min \{D, N \}$; however, they suffer from the significant challenges associated with non-convex optimization.  On the other hand, problems in the product space can be formulated to be convex, which affords many practical algorithms and analysis techniques, but one is required to optimize over a potentially large number of variables and solve a second factorization problem in order to recover the factors $(U,V)$ from the solution $X$.  These various pros and cons are briefly summarized in Table \ref{table:pro_con}.

\begin{table}[ht]
\caption[Typical properties of problems in the factorized vs product space.]{Typical properties of problems in the factorized vs product space. (Items in bold are desirable.)}
\vspace{-1mm}
\label{table:pro_con}
\centering
\begin{tabular}{c | c | c}
\hline
   & Product Space $(X)$ &  Factorized Space $(U,V)$ \\
\hline
Convex  & \textbf{Yes} & No \\
Problem Size & Large & \textbf{Small} \\
Structured Factors & No & \textbf{Yes} \\
\hline
\end{tabular}
\end{table}

To bridge this gap between the two classes of problems, here we explore the link between non-convex matrix factorization problems, which have the general form
\begin{equation}
	\label{eq:gen_fact}
	\textnormal{\textbf{Factorized Problems:}} \ \ \min_{U,V} \ell(Y,UV^T) + \lambda \Theta(U,V),
\end{equation}
and a closely related family of convex problems in the product space, given by
\begin{equation}
	\label{eq:gen_convex}
	\textnormal{\textbf{Convex Problems:}} \ \ \min_{X} \ell(Y,X) + \lambda \Omega_{\Theta}(X),
\end{equation}
where the function $\Omega_{\Theta}$ will be defined based on the choice of the regularization function $\Theta$ and will have the desirable property of being a convex function of $X$.  Unfortunately, while the optimization problem in \eqref{eq:gen_convex} is convex w.r.t. $X$, it will typically be non-tractable to solve.  Moreover, even if a solution to \eqref{eq:gen_convex} could be found, solving a convex problem in the product space does not necessarily achieve our goal, as we still must solve another matrix factorization problem to recover the $(U,V)$ factors with the desired properties encouraged by the $\Theta$ function (sparsity, non-negativity, etc.).  Nevertheless, the two problems given by \eqref{eq:gen_fact} and \eqref{eq:gen_convex} will be tightly coupled.  Specifically, the convex problem in \eqref{eq:gen_convex} will be shown to be a global lower-bound to the non-convex factorized problem in \eqref{eq:gen_fact}, and solutions $(U,V)$ to the factorized problem will yield solutions $X=UV^T$ to the convex problem.  As a result, we will tailor our results to the non-convex factorization problem \eqref{eq:gen_fact} using the convex problem \eqref{eq:gen_convex} as an analysis tool. While the optimization problem in the factorized space is not convex, by analyzing this tight interconnection between the two problems, we will show that if the number of columns in $(U,V)$ is large enough and can be adapted to the data instead of being fixed \textit{a priori}, local minima of the non-convex factorized problem will be global minima of both the convex and non-convex problems. This result will lead to a practical optimization strategy that is parallelizable and often requires a much smaller set of variables. Experiments in image processing applications will illustrate the effectiveness of the proposed approach.

%% file: PAMI16_StructMatFact_background.tex
\section{Mathematical Background \& Prior Work}  
\label{sec:PTN}

As discussed before, relaxing low-rank matrix factorization problems via nuclear norm formulations fails to capture the full generality of factorized problems as it does not yield ``structured" factors, $(U,V)$, with the desired properties encouraged by $\Theta(U,V)$ (sparseness, non-negativity, etc.).  To address this issue, several studies have explored a more general convex relaxation via the matrix norm given by
\begin{equation}
\label{eq:tensor_norm}
\begin{split}
\|X\|_{u,v} &\equiv \inf_{r \in \Nplus} \ \inf_{U,V : UV^T=X} \sum_{i=1}^r \|U_i\|_u \|V_i\|_v \\
	&\equiv \inf_{r \in \Nplus} \ \inf_{U,V : UV^T=X} \sum_{i=1}^r \tfrac{1}{2} (\|U_i\|_u^2+\|V_i\|_v^2) ,
\end{split}
\end{equation}
where $(U_i,V_i)$ denote the $i$\ts{th} columns of $U$ and $V$, respectively, $\|\cdot\|_u$ and $\|\cdot\|_v$ are arbitrary vector norms, and the number of columns, $r$, in the $U$ and $V$ matrices is allowed to be variable \cite{Ryan:2002,bach08,bach13,Haeffele:ICML14,Chandrasekaran:FoundCompMath2012}.  The norm in \eqref{eq:tensor_norm} has appeared under multiple names in the literature, including the projective tensor norm, decomposition norm, and atomic norm. It is worth noting that for particular choices of the $\|\cdot\|_u$ and $\|\cdot\|_v$ vector norms, $\|X\|_{u,v}$ reverts to several well known matrix norms and thus provides a generalization of many commonly used regularizers.  Notably, when the vector norms are both $l_2$ norms, the form in \eqref{eq:tensor_norm} becomes the well known variational definition of the nuclear norm \cite{Recht:SIAM10}:
\begin{align}
	\|X\|_* = \|X\|_{2,2} &\equiv \inf_{r \in \Nplus} \ \inf_{U,V : UV^T=X} \sum_{i=1}^r \|U_i\|_2 \|V_i\|_2 \\
	&\equiv \inf_{r \in \Nplus} \ \inf_{U,V : UV^T=X} \sum_{i=1}^r \tfrac{1}{2} (\|U_i\|_2^2+\|V_i\|_2^2).\nonumber
\end{align}
Moreover, by replacing the column norms in \eqref{eq:tensor_norm} with gauge functions one can incorporate additional regularization on $(U,V)$, such as non-negativity, while still being a convex function of $X$ \citep{bach13}. Recall that given a closed, convex set $C$ containing the origin, a gauge function, $\sigma_C(x)$, is defined as,
\begin{equation}
	\sigma_C(x) \equiv \inf_\mu \mu \st \mu x \in C.
\end{equation}
Further, recall that all norms are gauge functions (as can be observed by choosing $C$ to be the unit ball of a norm), but gauge functions are a slight generalization of norms, since they satisfy all the properties of a norm except that they are not required to be invariant to non-negative scaling (i.e., it is not required that $\sigma_C(x)$ be equal to $\sigma_C(-x)$).  Finally, note that given a gauge function, $\sigma_C$, its polar, $\sigma_C^\circ$, is defined as
\begin{equation}
		\sigma_C^\circ(z) \equiv \sup_x \left<z,x\right> \st \sigma_C(x)\leq 1,
\end{equation}
which itself is also a gauge function. In the case of a norm, its polar function is often referred to as the dual norm.

\subsection{Matrix Factorization as Semidefinite Optimization}
\label{sec:SDP}
Due to the increased modeling opportunities it provides, several studies have explored structured matrix factorization formulations based on the $\|\cdot\|_{u,v}$ norm in a way that allows one to work with a highly reduced set of variables while still providing some guarantees of global optimality.  In particular, it is possible to explore optimization problems over factorized matrices $(U,V)$ of the form
\begin{equation}
\label{eq:tensor_norm_obj}
\min_{U,V} \ell(Y,UV^T)+\lambda \|UV^T\|_{u,v}.
\end{equation}

While this problem is convex with respect to the product $X=UV^T$, it is still non-convex with respect to $(U,V)$ jointly due to the matrix product. However, if we define a matrix $\Gamma$ to be the concatenation of $U$ and $V$
\begin{equation}
\Gamma \equiv \begin{bmatrix} U \\ V \end{bmatrix} \implies \Gamma \Gamma^T= \begin{bmatrix} UU^T & UV^T \\ VU^T & VV^T \end{bmatrix},
\end{equation}
we see that $UV^T$ is a submatrix of the positive semidefinite matrix $M = \Gamma \Gamma^T$.  After defining the function $H : S_n^+ \rightarrow \Re$
\begin{equation}
\label{semidef_obj}
H(\Gamma \Gamma^T) = \ell(Y,UV^T) + \lambda \|UV^T\|_{u,v},
\end{equation}
it is clear that the proposed formulation \eqref{eq:tensor_norm_obj} can be recast as an optimization problem over a positive semidefinite matrix.\!

At first the above discussion seems to be a circular argument, since while $H(M)$ is a convex function of $M$, this says nothing about finding $\Gamma$ (or $U$ and $V$). 
However, results for semidefinite programs in standard form \cite{burer05} show that one can minimize $H(M)$ by solving for $\Gamma$ directly without introducing any additional local minima, provided that the rank of $\Gamma$ is larger than the rank of the true solution, $\hat M$.  Further, if the rank of $\hat M$ is not known \textit{a priori} and $H(M)$ is twice differentiable, then any local minima w.r.t. $\Gamma$ such that $\Gamma$ is rank-deficient give a global minimum of $H(\Gamma \Gamma^T)$ \cite{bach08}.

Unfortunately, many objective functions (such as those with the projective tensor norm) are not twice differentiable in general, so this result can not be applied directly. 
Nonetheless, if $H(M)$ is the sum of a twice differentiable and a non-differentiable convex function, then our prior work  \cite{Haeffele:ICML14} has shown that it is still possible to guarantee that rank-deficient local minima w.r.t. $\Gamma$ give global minima of $H(\Gamma \Gamma^T)$.  In particular, 
the result from \cite{bach08} can be extended to non-differentiable functions as follows.

\begin{proposition} \cite{Haeffele:ICML14}
Let $H \equiv F + G$, where $F : S_n^+ \rightarrow \Re$ is a twice differentiable convex function with compact level sets and $G : S_n^+ \rightarrow \Re$ is a proper, lower semi-continuous convex function that is potentially non-differentiable.  If $\tilde \Gamma$ is a rank deficient local minimum of $h(\Gamma) = H(\Gamma \Gamma^T)$, then $\hat M = \tilde \Gamma \tilde \Gamma^T$ is a global minimum of $H(M)$.
\end{proposition}
These results allow one to solve \eqref{eq:tensor_norm_obj} using a potentially highly reduced set of variables if the rank of the true solution is much smaller than the dimensionality of $M$.

However, while the above results from semidefinite programming 
are sufficient if we only wish to find general factors such that $UV^T=X$, for the purposes of solving structured matrix factorizations, we are interested in finding factors $(U,V)$ that achieve the infimum in the definition of \eqref{eq:tensor_norm}, which is not provided by a solution to \eqref{eq:tensor_norm_obj}, as the $(U,V)$ that one obtains is not necessarily the $(U,V)$ that minimize \eqref{eq:tensor_norm}. 
As a result, the results from semidefinite optimization are not directly applicable to problems such as \eqref{eq:tensor_norm_obj} as they deal with different optimization problems. In the remainder of this paper, we will show that results regarding global optimality can still be derived for the non-convex optimization problem given in \eqref{eq:tensor_norm_obj} as well as for more general matrix factorization formulations.

\section{\!\!\!Structured Matrix Factorization Problem}

To develop our analysis we will introduce a matrix regularization function, $\Omega_{\theta}$, that generalizes the $\|\cdot\|_{u,v}$ norm and is similarly defined in the product space but allows one to enforce structure in the factorized space. We will establish basic properties of the proposed regularization function, such as its convexity, and discuss several practical examples.

\subsection{Structured Matrix Factorization Regularizers}
\label{sec:SMF-regularizer}
The proposed matrix regularization function, $\Omega_{\theta}$, will be constructed from a regularization function $\theta$ on rank-1 matrices, which can be defined as follows:
\begin{definition}
	\label{def:rank1_meas}
	A function $\theta : \Re^D \times \Re^N \rightarrow \Re_+ \cup \infty$ is said to be a \textit{\textbf{rank-1 regularizer}} if
	\begin{enumerate}
		\item $\theta(u,v)$ is positively homogeneous with degree 2, i.e., $\theta(\alpha u,\alpha v) = \alpha^2 \theta(u,v) \ \forall \alpha \geq 0, \forall (u,v)$.
		\item $\theta(u,v)$ is positive semi-definite, i.e., $\theta(0,0)=0$ and $\theta(u,v) \geq 0 \ \forall(u,v)$.
		\item For any sequence $(u_n,v_n)$ such that $\|u_n v_n^T\| \rightarrow \infty$ we have that $\theta(u_n,v_n) \rightarrow \infty$.
	\end{enumerate}
\end{definition}
The first two properties of the definition are straight-forward and their necessity will become apparent when we derive properties of $\Omega_\theta$.  The final property is necessary to ensure that $\theta$ is well defined as a regularizer for rank-1 matrices. For example, taking $\theta(u,v) = \|u\|^2 + \delta_+(v)$, where $\delta_+$ denotes the non-negative indicator function, satisfies the first two properties of the definition, but not the third since $\theta(u,v)$ can always be decreased by taking $u\rightarrow \alpha u$, $v \rightarrow \alpha^{-1} v$ for some $\alpha \in (0,1)$ without changing the value of $(\alpha u)(\alpha^{-1} v)^T = uv^T$.  Note also that the third property implies that $\theta(u,v) > 0$ for all $uv^T \neq 0$.

These three properties define a general set of requirements that are satisfied by a very wide range of rank-1 regularizers (see \S\ref{sec:structured-factorization-examples} for specific examples of regularizers that can be used for well known problems). While we will prove our theoretical results using this general definition of a rank-1 regularizer, later, when discussing specific algorithms that can be used to solve structured matrix factorization problems in practice, we will require that $\theta(u,v)$ satisfies a few additional requirements.

Using the notion of a rank-1 regularizer, we now define a regularization function on matrices of arbitrary rank:
\begin{definition}
	\label{def:mat_fact_reg}
	Given a rank-1 regularizer $\theta : \Re^D \times \Re^N \rightarrow \Re_+ \cup \infty$, the \textit{\textbf{matrix factorization regularizer}} $\Omega_\theta : \Re^{D \times N} \rightarrow \Re_+ \cup \infty$ is defined as
	\begin{equation}
	\label{eq:omega_mat_def}
	\!\!\!
		\Omega_\theta(X) \equiv \inf_{r \in \Nplus} \inf_{\substack{U \in \Re^{D \times r} \\ V \in \Re^{N \times r}}} \sum_{i=1}^r \theta(U_i,V_i) \st X = UV^T.
	\end{equation}
\end{definition}
The function defined in \eqref{eq:omega_mat_def} is very closely related to other regularizers that have appeared in the literature.  In particular, taking $\theta(u,v) = \|u\|_u \|v\|_v$ or $\theta(u,v) = \tfrac{1}{2}(\|u\|_u^2 + \|v\|_v^2)$ for arbitrary vector norms $\|\cdot\|_u$ and $\|\cdot\|_v$ gives the $\|\cdot\|_{u,v}$ norm in \eqref{eq:tensor_norm}.  Note, however, there is no requirement for $\theta(u,v)$ to be convex w.r.t. $(u,v)$ or to be composed of norms.  

\subsection{Properties of the Matrix Factorization Regularizer}

As long as $\theta$ satisfies the requirements from Definition \ref{def:rank1_meas}, one can show that $\Omega_{\theta}$ satisfies the following proposition:
\begin{proposition}
\label{prop:omega_prop}
Given a rank-1 regularizer $\theta$, the matrix facto\-rization regularizer $\Omega_{\theta}$ satisfies the following properties:
\begin{enumerate}
	\item $\Omega_{\theta}(0) = 0$ and $\Omega_{\theta}(X) > 0 \ \forall X \neq 0$.
	\item $\Omega_{\theta}(\alpha X) = \alpha \Omega_{\theta}(X) \ \forall \alpha \geq 0, \ \forall X$.
	\item $\Omega_{\theta}(X+Z) \leq \Omega_{\theta} (X) + \Omega_{\theta} (Z) \ \forall (X,Z)$.
	\item $\Omega_{\theta}(X)$ is convex w.r.t. $X$.
	\item The infimum in \eqref{eq:omega_mat_def} is achieved with $r \leq DN$.
	\item \label{prop:omega_is_a_norm}
	If $\theta(-u,v)=\theta(u,v)$ or $\theta(u,-v)=\theta(u,v)$, then $\Omega_{\theta}$ is a norm on $X$.
	\item The subgradient $\partial \Omega_{\theta}(X)$ of $\Omega_{\theta}(X)$ is given by:
		\label{prop:omega_subgrad}
\begin{equation*}
\!\!\!\!
		\big\{ W \!:\! \langle W,X \rangle \!=\! \Omega_{\theta}(X), u^T W v \leq \theta(u,v) \ \forall(u,v) \big\}.\!\!\!
\end{equation*}

\item Given a factorization $X=UV^T$, if there exists a matrix $W$ such that $\sum_{i=1}^r U_i^T W V_i = \sum_{i=1}^r \theta(U_i,V_i)$ and $u^T W v \leq \theta(u,v)$ $\forall(u,v)$, then $UV^T$ is an optimal factorization of $X$, i.e., it achieves the infimum in \eqref{eq:omega_mat_def}, and $W \in \partial \Omega_\theta(X)$. 

\end{enumerate}
\end{proposition}

\begin{proof}
A full proof of the above Proposition can be found in Appendix \ref{sec:omega_prop_proof} and uses similar arguments to those found in  \cite{Ryan:2002,bach08,bach13,Chandrasekaran:FoundCompMath2012,Yu:2014} for related problems.
\end{proof}

Note that the first 3 properties show that $\Omega_{\theta}$ is a gauge function on $X$ (and further it will be a norm if property \ref{prop:omega_is_a_norm} is satisfied). While this also implies that $\Omega_{\theta}(X)$ must be a convex function of $X$, note that it can still be very challenging to evaluate or optimize functions involving $\Omega_{\theta}$ due to the fact that it requires solving a non-convex optimization problem by definition.  However, by exploiting the convexity of $\Omega_{\theta}$, we are able to use it to study the optimality conditions of many associated non-convex matrix factorization problems, some of which are discussed next.

\subsection{Structured Matrix Factorization Problem Examples}
\label{sec:structured-factorization-examples}
The matrix factorization regularizer provides a natural bridge between convex formulations in the product space \eqref{eq:gen_convex} and non-convex functions in the factorized space \eqref{eq:gen_fact} due to the fact that $\Omega_{\theta}(X)$ is a convex function of $X$, while from the definition \eqref{eq:omega_mat_def} one can induce a wide range of properties in $(U,V)$ by an appropriate choice of $\theta(u,v)$ function.  In what follows, we give a number of examples which lead to variants of several structured matrix factorization problems that have been studied previously in the literature.

\textbf{Low-Rank}: The first example of note is to relax low-rank constraints into nuclear norm regularized problems.  Taking $\theta(u,v) = \tfrac{1}{2}(\|u\|_2^2 + \|v\|_2^2)$ gives the well known variational form of the nuclear norm, $\Omega_{\theta}(X) = \|X\|_*$, and thus provides a means to solve problems in the factorized space where the size of the factorization gets controlled by regularization.  In particular, we have the conversion
\begin{equation}
\begin{split}
	&\min_X \ell(Y,X) + \lambda \|X\|_* \iff \\
	&\min_{r,U,V} \ell(Y,UV^T) + \tfrac{\lambda}{2} \sum_{i=1}^r(\|U_i\|_2^2 + \|V_i\|_2^2) \iff \\
	&\min_{r,U,V} \ell(Y,UV^T) + \lambda \sum_{i=1}^r \|U_i\|_2\|V_i\|_2,
\end{split}
\end{equation}
where the $\iff$ notation implies that solutions to all 3 objective functions will have identical values at the global minimum and any global minimum w.r.t. $(U,V)$ will be a global minimum for $X=UV^T$. While the above equivalence is well known for the nuclear norm \citep{Cabral:ICCV2013,Recht:SIAM10}, the factorization is ``unstructured" in the sense that the Euclidean norms do not bias the columns of $U$ and $V$ to have any particular properties. Therefore, to find factors with additional structure, such as non-negativity, sparseness, etc., more general $\theta(u,v)$ functions need to be considered.

\textbf{Non-Negative Matrix Factorization}: 
If we extend the previous example to now add non-negative constraints on $(u,v)$, we get $\theta(u,v) = \tfrac{1}{2} (\|u\|_2^2+\|v\|_2^2) + \delta_{\Re_+}(u) + \delta_{\Re_+}(v)$, where $\delta_{\Re_+}(u)$ denotes the indicator function that $u\geq 0$. This choice of $\theta$ acts similarly to the variational form of the nuclear norm in the sense that it limits the number of non-zero columns in $(U,V)$, but it also imposes the constraints that $U$ and $V$ must be non-negative.  As a result, one gets a convex relaxation of traditional non-negative matrix factorization
\begin{align}
	& \min_{U,V}  \ell(Y,UV^T) \st U \geq 0, \ V \geq 0 \implies \\
	&\min_{r,U,V} \ell(Y,UV^T) + \tfrac{\lambda}{2}\sum_{i=1}^r(\|U_i\|_2^2+\|V_i\|_2^2) \st U \geq 0, \ V \geq 0. \nonumber
\end{align}
The $\implies$ notation is meant to imply that the two problems are not strictly equivalent as in the nuclear norm example.  The key difference between the two problems is that in the first one the number of columns in $(U,V)$ is fixed \textit{a priori}, while in the second one the number of columns in $(U,V)$ is allowed to vary and is adapted to the data via the low-rank regularization induced by the Frobenius norms on $(U,V)$.

\textbf{Row or Columns Norms}: Taking $\theta(u,v) = \|u\|_1 \|v\|_v$ results in $\Omega_{\theta}(X) = \sum_{i=1}^D \|(X^T)_i\|_v$, i.e., the sum of the $\|\cdot\|_v$ norms of the rows of $X$, while taking $\theta(u,v) =  \|u\|_u \|v\|_1$ results in $\Omega_{\theta}(X) = \sum_{i=1}^N \|X_i\|_u$, i.e., the sum of the $\|\cdot\|_u$ norms of the columns of $X$ \cite{bach08,Ryan:2002}.  As a result, the regularizer $\Omega_{\theta}(X)$ generalizes the $\|X\|_{u,1}$ and $\|X\|_{1,v}$ mixed norms. The reformulations into a factorized form give:
\begin{align*}
&\min_X \ell(Y,X) \!+\! \lambda \|X\|_{1,v} \! \Leftrightarrow \! \min_{r,U,V} \! \ell(Y,UV^T) \!+\! \lambda \sum_{i=1}^r \! \|U_i\|_1 \|V_i\|_v \\
&\min_X \ell(Y,X) \!+\! \lambda \|X\|_{u,1}  \! \Leftrightarrow \! \min_{r,U,V} \! \ell(Y,UV^T) \!+\! \lambda \sum_{i=1}^r \! \|U_i\|_u \|V_i\|_1.
\end{align*}
However, the factorization problems in this case are relatively uninteresting as taking either $U$ or $V$ to be the identity (depending on whether the $l_1$ norm is on the columns of $U$ or $V$, respectively) and the other matrix to be $X$ (or $X^T$) results in one of the possible optimal factorizations.  

\textbf{Sparse Dictionary Learning}: Similar to the non-negative matrix factorization case, convex relaxations of sparse dictionary learning can also be obtained by combining $l_2$ norms with sparsity-inducing regularization.  For example, taking $\theta(u,v) = \tfrac{1}{2}(\|u\|_2^2 + \|v\|_2^2+\gamma\|v\|_1^2)$ results in a relaxation
\begin{equation}
	\begin{split}
	&\min_{U,V} \ell(Y,UV^T)+\lambda \|V\|_1 \st \|U_i\|_2 = 1 \ \ \forall i \implies \\
	&\min_{r,U,V} \ell(Y,UV^T)+\tfrac{\lambda}{2}\sum_{i=1}^r(\|U_i\|_2^2 + \|V_i\|_2^2+ \gamma \|V_i\|_1^2),
	\end{split}
\end{equation}
which was considered as a potential formulation for sparse dictionary learning in \cite{bach08}, where now the number of atoms in the dictionary is fit to the dataset via the low-rank regularization induced by the Frobenius norms.  A similar approach would be to take $\theta(u,v) = \|u\|_2 (\|v\|_2+\gamma \|v\|_1)$.

\textbf{Sparse PCA}: If both the rows and columns of $U$ and $V$ are regularized to be sparse, then one can obtain convex relaxations of sparse PCA \cite{Zou:JCGS2006}.  One example of this is to take $\theta(u,v) = \tfrac{1}{2}(\|u\|_2^2 + \gamma_u \|u\|_1^2+ \|v\|_2^2+\gamma_v\|v\|_1^2)$.  Alternatively, one can also place constraints on the number of elements in the non-zero support of each column in $(u,v)$ via a rank-1 regularizer of the form $\theta(u,v) = \tfrac{1}{2}(\|u\|_2^2+\|v\|_2^2) + \delta_{\|\cdot\|_0 \leq k} (u) + \delta_{\|\cdot\|_0 \leq q} (v)$, where $\delta_{\|\cdot\|_0 \leq k}(u)$ denotes the indicator function that $u$ has $k$ or fewer non-zero elements.  Such a form was analyzed in \cite{Richard:NIPS2014} and gives a relaxation of sparse PCA that regularizes the number of sparse components via the $\ell_2$ norm, while requiring that a given component have the specified level of sparseness.

\textbf{General Structure}: More generally, this theme of using a combination of $l_2$ norms and additional regularization on the factors can be used to model additional forms of structure on the factors.  For example one can take $\theta(u,v) = \|u\|_2 \|v\|_2 + \gamma \hat{\theta}(u,v)$ or $\theta(u,v) = \|u\|_2^2 + \|v\|_2^2 + \gamma \hat{\theta}(u,v)$ with a function $\hat{\theta}$ that promotes the desired structure in $U$ and $V$ provided that $\theta(u,v)$ satisfies the necessary properties in the definition of a rank-1 regularizer.  Additional example problems can be found in \cite{bach13, Haeffele:ICML14}.

\textbf{Symmetric Factorizations}: Assuming that $X$ is a square matrix, it is also possible to learn symmetrical formulations with this framework, as the indicator function $\delta_{u=v}(u,v)$ that requires $u$ and $v$ to be equal is also positively homogeneous.  As a result, one can use regularization such as $\theta(u,v) = \delta_{u=v}(u,v) + \|u\|_2^2$ to learn low-rank symmetrical factorizations of $X$, and add additional regularization to encourage additional structures.  For example $\theta(u,v) = \delta_{u=v}(u,v) + \|u\|_2^2 + \|u\|_1^2 + \delta_{\Re_+}(u)$ learns symmetrical factorizations where the factors are required to be non-negative and encouraged to be sparse.

%% file: PAMI16_StructMatFact_factorization.tex
\section{Theoretical Analysis}

In this section we provide a theoretical analysis of the link between convex formulations \eqref{eq:gen_convex}, which offer guarantees of global optimality, and factorized formulations \eqref{eq:gen_fact}, which offer additional flexibility in modeling the data structure and recovery of features that can be used in subsequent analysis. Using the matrix factorization regularizer introduced in \S\ref{sec:SMF-regularizer}, we consider the convex optimization problem
\begin{equation}
\label{general_obj_nofact}
\min_{X,Q} \big\{ F(X,Q) \equiv \ell(Y,X,Q) + \lambda \Omega_{\theta}(X) \big\}.
\end{equation}
Here the term $Q$ can be used to model additional variables that will not be factorized. For example, in robust PCA (RPCA) \cite{Candes:ACM11} the term $Q$ accounts for sparse outlying entries, and a formulation in which the data is corrupted by both large corruptions and Gaussian noise is given by:
\begin{equation*}
	\min_{X,Q} \{ F(X,Q)_{RPCA} \equiv \tfrac{1}{2} \|Y-X-Q\|_F^2 + \gamma\|Q\|_1 + \lambda \|X\|_*\}.
\end{equation*}
In addition to the convex formulation \eqref{general_obj_nofact}, we will also consider the closely related non-convex factorized formulation
\begin{equation}
	\label{eq:mat_fact_obj}
	\min_{U,V,Q}  \big\{  f(U,V,Q) \equiv  \ell(Y,UV^T,Q) + \lambda \sum_{i=1}^r \theta(U_i,V_i)  \big\} . 
\end{equation}
Note that in \eqref{eq:mat_fact_obj} we consider problems where $r$ is held fixed. If we additionally optimize over the factorization size, $r$, in \eqref{eq:mat_fact_obj}, then the problem becomes equivalent to \eqref{general_obj_nofact}, as we shall see.  We will assume throughout that $\ell(Y,X,Q)$ is jointly convex w.r.t. $(X,Q)$ and once differentiable w.r.t. $X$. Also, we will use the notation $[r]$ to denote the set $\{1,\dots,r\}$.

\subsection{Conditions under which Local Minima Are Global}
\label{sec:main-results}

Given the non-convex optimization problem \eqref{eq:mat_fact_obj}, note from the definition of $\Omega_{\theta}(X)$ that for all $(X,U,V)$ such that $X=UV^T$ we must have $\Omega_{\theta}(X) \leq \sum_{i=1}^r \theta(U_i,V_i)$. This yields a global lower bound between the convex and non-convex objective functions, i.e., for all $(X,U,V)$ such that $X=UV^T$
\begin{align}
	\label{eq:global_bound}
	F(X,Q)
	&=\ell(Y,X,Q) + \lambda \Omega_{\theta}(X) \\
	& \leq \ell(Y,UV^T,Q) + \lambda \sum_{i=1}^r \theta(U_i,V_i)=f(U,V,Q).\nonumber
\end{align}
From this, if $\hat{X}$ denotes an optimal solution to the convex problem $\min_{X,Q} F(X,Q)$, then any factorization $UV^T=\hat{X}$ such that $\sum_{i=1}^r \theta(U_i,V_i) = \Omega_{\theta}(\hat{X})$ is also an optimal solution to the non-convex problem $\min_{U,V,Q} f(U,V,Q)$.  These properties lead to the following result.

\begin{theorem}
	\label{thm:mat_fact_0_min}
	Given a function $\ell(Y,X,Q)$ that is jointly convex in $(X,Q)$ and once differentiable w.r.t. $X$; a rank-1 regularizer $\theta$ that satisfies the conditions in Definition \ref{def:mat_fact_reg}; and a constant $\lambda > 0$, local minima $(\tilde U, \tilde V, \tilde Q)$ of $f(U,V,Q)$ in \eqref{eq:mat_fact_obj}
are globally optimal if $(\tilde U_i,\tilde V_i) = (0,0)$ for some $i \in [r]$. Moreover, $(\hat{X},\hat{Q}) = (\tilde U \tilde V^T, \tilde Q)$ is a global minima of $F(X,Q)$ and $\tilde U \tilde V^T$ is an optimal factorization of $\hat{X}$.
\end{theorem}

\begin{proof}
Since $F(X,Q)$ provides a global lower bound for
$f(U,V,Q)$, the result follows from the fact that local minima of $f(U,V,Q)$ that satisfy the conditions of the theorem also satisfy the conditions for global optimality of $F(X,Q)$. More specifically, because $F(X,Q)$ is a convex function, we have that $(\hat X,\hat Q)$ is a global minimum of $F(X,Q)$ iff
	\begin{equation}
			-\tfrac{1}{\lambda}\nabla_X \ell(Y,\hat X,\hat Q) \in \partial \Omega_{\theta}(X) ~~ \text{and} ~~
			0 \in \partial_Q \ell(Y,\hat X,\hat Q).
	\end{equation}
	Since $(\tilde U,\tilde V,\tilde Q)$ is a local minimum of $f(U,V,Q)$,
	it is necessary that $0 \in \partial_Q \ell(Y,\tilde U \tilde V^T,\tilde Q)$.
	Moreover, from the characterization of the subgradient of $\Omega_{\theta}(X)$ given in Proposition \ref{prop:omega_prop}, we also have that $-\tfrac{1}{\lambda}\nabla_X \ell(Y,X,Q) \in \partial \Omega_{\theta}(X)$ will be true for $X = \tilde U \tilde V^T$ if the following conditions are satisfied
	\begin{align}
		\label{eq:mat_fact_subgrad_cond1}
		&u^T(-\tfrac{1}{\lambda}\nabla_X \ell(Y,\tilde{U} \tilde{V}^T,\tilde Q))v \leq \theta(u,v) \ \forall(u,v) \\
		\label{eq:mat_fact_subgrad_cond2}
		&\sum_{i=1}^r \tilde{U}_i^T(-\tfrac{1}{\lambda}\nabla_X \ell(Y,\tilde{U} \tilde{V}^T,\tilde Q)) \tilde{V}_i = \sum_{i=1}^r \theta(\tilde{U}_i,\tilde{V}_i).
	\end{align}
				
To show \eqref{eq:mat_fact_subgrad_cond1}, recall that the local minimum $(\tilde U,\tilde V,\tilde Q)$ is such that one column pair of $(\tilde U,\tilde V)$ is 0. Assume without loss of generality that the final column pair of $(\tilde U,\tilde V)$ is 0 and let $U_{\epsilon} = [\tilde U_1,\ldots,\tilde U_{r-1}, \epsilon^{1/2} u]$ and $V_{\epsilon} = [\tilde V_1,\ldots,\tilde V_{r-1}, \epsilon^{1/2} v]$ for some $\epsilon > 0$ and arbitrary $(u,v)$. Then, due to the fact that $(\tilde U,\tilde V,\tilde Q)$ is a local minimum, we have that for all $(u,v)$ there exists $\delta >0$ such that for all $\epsilon \in (0,\delta)$ we have
	\begin{align}
	\label{eqn:UepsVeps}
	&\ell(Y,U_{\epsilon}V_{\epsilon}^T,\tilde Q) + \lambda \sum_{i=1}^{r} \theta(\tilde U_i,\tilde V_i)+\lambda \theta(\epsilon^{1/2}u,\epsilon^{1/2}v) = \\
	\label{eqn:UepsVepsposhom}
	&\ell(Y,\tilde U \tilde V^T + \epsilon uv^T,\tilde Q) + \lambda \sum_{i=1}^r \theta(\tilde U_i,\tilde V_i) + \epsilon \lambda \theta(u,v) \geq \\
	&\ell(Y,\tilde U \tilde V^T,\tilde Q) + \lambda \sum_{i=1}^r \theta(\tilde U_i,\tilde V_i),
	\end{align}
where the equivalence between \eqref{eqn:UepsVeps} and \eqref{eqn:UepsVepsposhom} follows from the positive homogeneity of $\theta$. Rearranging terms, we have
\begin{equation}
	\tfrac{-1}{\lambda\epsilon}[\ell(Y,\tilde U \tilde V^T+\epsilon uv^T,\tilde Q) - \ell(Y,\tilde U \tilde V^T,\tilde Q)] \leq \theta(u,v).
\end{equation}
Since $\ell(Y,X,Q)$ is differentiable w.r.t. $X$, after taking the limit w.r.t. $\epsilon \searrow 0$, we obtain $\left< \tfrac{-1}{\lambda}\nabla_X \ell(Y,\tilde U \tilde V^T,\tilde Q),uv^T \right> \leq \theta(u,v)$ for any $(u,v)$ vector pair, showing \eqref{eq:mat_fact_subgrad_cond1}. 

To show \eqref{eq:mat_fact_subgrad_cond2}, let $U_{1\pm\epsilon} = (1\pm\epsilon)^{1/2}\tilde U$ and $V_{1\pm\epsilon} = (1\pm\epsilon)^{1/2} \tilde V$ for some $\epsilon > 0$. Since $(\tilde U, \tilde V, \tilde Q)$ is a local minimum, there exists  $\delta>0$ such that for all $\epsilon \in (0,\delta)$ we have
\begin{equation}
	\begin{split}
	\ell(Y,& U_{1 \pm \epsilon} V_{1 \pm \epsilon}^T, \tilde Q) + \lambda \sum_{i=1}^r\theta((1 \pm \epsilon)^{1/2} \tilde U_i,(1 \pm \epsilon)^{1/2} \tilde V_i) \\
	&= \ell(Y,(1\pm \epsilon) \tilde U \tilde V^T,\tilde Q) + \lambda (1 \pm \epsilon) \sum_{i=1}^r \theta(\tilde U_i,\tilde V_i) \\
	& \geq \ell(Y,\tilde U \tilde V^T,\tilde Q) + \lambda \sum_{i=1}^r \theta (\tilde U_i,\tilde V_i) .
	\end{split}
	\!\!\!
\end{equation}
Rearranging terms gives
\begin{equation*}
	\tfrac{-1}{\lambda\epsilon}[\ell(Y, (1\pm \epsilon) \tilde U \tilde V^T,\tilde Q) - \ell(Y,\tilde U \tilde V^T,\tilde Q)] \leq \pm \sum_{i=1}^r \theta(\tilde U_i, \tilde V_i),
\end{equation*}
and taking the limit w.r.t. $\epsilon \searrow 0$ gives
\begin{align*}
\sum_{i=1}^r \theta(\tilde U_i,\tilde V_i) \leq \left<\tfrac{-1}{\lambda} \nabla_X \ell(Y,\tilde U \tilde V^T,\tilde Q),\tilde U \tilde V^T \right> \leq \sum_{i=1}^r \theta(\tilde U_i,\tilde V_i),
\end{align*}
showing \eqref{eq:mat_fact_subgrad_cond2}. The last two statements of the result follow from the discussion before the theorem, together with Proposition \ref{prop:omega_prop} part (8).
\end{proof}

Note that the above proof provides sufficient conditions to guarantee the global optimality of local minima with specific properties, but in addition it also proves the following sufficient conditions for global optimality of any point.
\begin{corollary}
\label{cor:global_min_conds}
Given a function $\ell(Y,X,Q)$ that is jointly convex in $(X,Q)$ and once differentiable w.r.t. $X$; a rank-1 regularizer $\theta$ that satisfies the conditions in Definition \ref{def:mat_fact_reg}; and a constant $\lambda > 0$, a point $(\tilde{U},\tilde{V},\tilde{Q})$ is a global minimum of $f(U,V,Q)$ in \eqref{eq:mat_fact_obj}
	if it satisfies the conditions
	\begin{enumerate}
		\item \label{cond:Q} $0 \in \partial_Q \ell(Y,\tilde{U}\tilde{V}^T,\tilde{Q})$
		\item \label{cond:UV1} $\tilde{U}_i^T (\tfrac{-1}{\lambda} \nabla_X \ell(Y,\tilde{U}\tilde{V}^T,\tilde{Q})) \tilde{V}_i = \theta(\tilde{U}_i,\tilde{V}_i) \ \forall i \in [r]$
		\item \label{cond:UV2} $u^T (\tfrac{-1}{\lambda} \nabla_X \ell(Y,\tilde{U}\tilde{V}^T,\tilde{Q})) v \leq \theta(u,v) \ \forall (u,v)$.
	\end{enumerate}
\end{corollary}
Condition \ref{cond:Q} is easy to verify, as one can hold $(U,V)$ constant and solve a convex optimization problem for $Q$.  Likewise, condition \ref{cond:UV1} is simple to test (note that the condition is equivalent to \eqref{eq:mat_fact_subgrad_cond2} due to the bound in \eqref{eq:mat_fact_subgrad_cond1}), and if a $(U_i,V_i)$ pair exists which does not satisfy the equality, then one can decrease the objective function by scaling $(U_i,V_i)$ by a non-negative constant.  Further, for many problems, it is possible to show that points that satisfy first-order optimality will satisfy conditions \ref{cond:Q} and \ref{cond:UV1}, such as in the following result.
\begin{proposition}
\label{prop:1st_order}
Given a function $\ell(Y,X,Q)$ that is jointly convex in $(X,Q)$ and once differentiable w.r.t. $X$; a constant $\lambda > 0$; and two gauge functions $(\sigma_u(u),\sigma_v(v))$, then for $\theta(u,v) = \sigma_u(u) \sigma_v(v)$ or $\theta(u,v) = \tfrac{1}{2} (\sigma_u(u)^2 + \sigma_v(v)^2)$, any first-order optimal point $(\tilde{U},\tilde{V},\tilde{Q})$ of  $f(U,V,Q)$ in \eqref{eq:mat_fact_obj}
	satisfies conditions \ref{cond:Q}-\ref{cond:UV1} of Corollary \ref{cor:global_min_conds}.
\end{proposition}
\begin{proof} See Appendix \ref{sec:1st_order_proof}. \end{proof}

As many optimization algorithms can guarantee convergence to first-order optimal points, from the above result and discussion it is clear that the primary challenge in verifying if a given point is globally optimal is to test if condition  \ref{cond:UV2} of Corollary \ref{cor:global_min_conds} holds true.  This is known as the polar problem and is discussed in detail next. 

\subsection{Ensuring Local Minimality via the Polar Problem}
\label{subsec:polar_prob}
Note that, because the optimization problem in \eqref{eq:mat_fact_obj} is non-convex, first-order optimality is not sufficient to guarantee a local minimum. Thus, to apply the results from Section~\ref{sec:main-results} in practice one needs to verify that condition \ref{cond:UV2} from Corollary \ref{cor:global_min_conds} is satisfied. This problem is known as the polar problem and generalizes the concept of a dual norm. In particular, the proof of Proposition \ref{prop:omega_prop} shows that the polar function $\Omega_{\theta}^{\circ} (Z) \equiv \sup_X \langle Z, X \rangle \st \Omega_{\theta}(X) \leq 1$ of a given matrix factorization regularizer $\Omega_{\theta}(X)$ can be computed as
\begin{equation}
	\label{eq:mat_polar}
		\Omega_{\theta}^{\circ}(Z) = \sup_{u,v} u^T Z v \ST \theta(u,v) \leq 1.
\end{equation}
Therefore, condition \ref{cond:UV2} of Corollary \ref{cor:global_min_conds} is equivalent to $\Omega_{\theta}^{\circ}(\tfrac{-1}{\lambda} \nabla_X \ell(Y,\tilde{U}\tilde{V}^T,\tilde{Q})) \leq 1$.  Note that the difficulty of solving the polar problem heavily depends on the particular choice of the $\theta$ function.  For example for $\theta(u,v) = \|u\|_1 \|v\|_1$ the polar problem reduces to simply finding the largest entry of $Z$ in absolute value, while for $\theta(u,v) = \|u\|_{\infty}\|v\|_{\infty}$ solving the polar problem is known to be NP-hard \cite{Hendrickx:SIAM2010}.

While for general $\theta(u,v)$ functions it is not necessarily known how to efficiently solve the polar problem, given a point $(\tilde{U},\tilde{V},\tilde{Q})$ that satisfies conditions \ref{cond:Q} and \ref{cond:UV1} of Corollary \ref{cor:global_min_conds}, the value of the polar problem solution at a given point and how closely the polar problem can be approximated provide a bound on how far a particular point is from being globally optimal.  The bound is based on the following result.\!

\begin{proposition}
\label{prop:approx_bound}
	Given a function $\ell(Y,X,Q)$ that is lower-semicontinuous, jointly convex in $(X,Q)$, and once differentiable w.r.t. $X$; a rank-1 regularizer $\theta$ that satisfies the conditions in Definition \ref{def:mat_fact_reg}; and a constant $\lambda > 0$, 
	for 	any point $(\tilde{U},\tilde{V},\tilde{Q})$ that satisfies conditions \ref{cond:Q} and \ref{cond:UV1} of Corollary \ref{cor:global_min_conds}, we have the following bound
		\begin{equation}
			\begin{split}
			f ( \tilde{U}, \tilde{V}, &~  \tilde{Q}) - F(\hat{X},\hat{Q}) \leq \\
			& \lambda \Omega_{\theta}(\hat{X}) [\Omega_{\theta}^{\circ}(\tfrac{-1}{\lambda}\nabla_X \ell(Y,\tilde{U}\tilde{V}^T,\tilde{Q})) - 1] \\
			&-  \tfrac{m_X}{2}\|\tilde{U}\tilde{V}^T-\hat{X}\|_F^2 - \tfrac{m_Q}{2} \| \tilde{Q}-\hat{Q} \|_F^2 , 
			\end{split}
		\end{equation}
		where $(\hat{X},\hat{Q})$ denotes a global minimizer of $F(X,Q)$ in \eqref{general_obj_nofact}, and $m_X \geq 0$ and $m_Q \geq 0$ denote the constants of strong-convexity of $\ell$ w.r.t. $X$ and $Q$, respectively (note that both $m$ constants can be 0 if $\ell$ is not strongly convex).
\end{proposition}
\begin{proof} See Appendix \ref{sec:approx_bound_proof}. \end{proof}

There are a few interpretations one can draw from the above proposition.  First, if $\hat{X}=0$ is a solution to $\min_{X,Q} F(X,Q)$, then the only $(U,V)$ pair that will satisfy conditions \ref{cond:Q} and \ref{cond:UV1} of Corollary \ref{cor:global_min_conds} is a global optimum (and if $m_X > 0$ then the solution is the unique solution $UV^T = 0$).  Second, for $\hat{X} \neq 0$ recall that $\Omega_{\theta}(\hat{X}) > 0$ and $\Omega_{\theta}^{\circ}(\tfrac{-1}{\lambda} \nabla_X \ell(Y,\tilde{U}\tilde{V}^T,\tilde{Q})) \geq 1$, since if condition \ref{cond:UV1} of Corollary \ref{cor:global_min_conds} is satisfied then the polar is clearly at least equal to 1 by definition of the polar.  If the polar is truly greater than 1, then the $[\Omega_{\theta}^{\circ}(\tfrac{-1}{\lambda}\nabla_X \ell(Y,\tilde{U}\tilde{V}^T,\tilde{Q})) - 1]$ in the above proposition effectively measures the error between the true value of the polar and our lower-bound estimate of the polar.  Further, the maximum difference between a first-order optimal point and the global minimum is upper bounded by the value of the polar at that point, and if the loss function $\ell$ is strongly convex, the error in the objective function is decreased further.  As a result, if one can guarantee solutions to the polar problem to within a given error level or provide an upper-bound on the polar problem, one can also guarantee solutions that are within a given error level of the global optimum.

A final interpretation of Proposition \ref{prop:approx_bound} is to note that the final condition of Corollary \ref{cor:global_min_conds} (and the need for an all-zero column in Theorem \ref{thm:mat_fact_0_min}) is essentially a check that the size of the representation (i.e., the number of columns $r$ in $U$ and $V$) is sufficiently large to represent the global optimum.  If, instead, we find a local minimum with a smaller representation than the global optimum, $r < \hat{r}$, where $\hat{r}$ denotes the number of columns in the global optimum, then the value of the $[\Omega_{\theta}^{\circ}(\tfrac{-1}{\lambda}\nabla_X \ell(Y,\tilde{U}\tilde{V}^T,\tilde{Q})) - 1]$ bounds how far from the global minimum we are by using a more compact representation (i.e., using only $r$ instead of $\hat{r}$ columns).

As a concrete example, consider the case where $\theta(u,v) = \|u\|_2 \|v\|_2$.  Recall that this choice of $\theta$ gives the nuclear norm, $\Omega_{\theta}(X) = \|X\|_*$, whose polar function is given by
\begin{align}
	\|Z\|_*^{\circ} &= \sup_{u,v} \{ u^T Z v \st \|u\|_2 \|v\|_2 \leq 1\} \\
		&= \sup_{u,v} \{ u^T Z v \st \|u\|_2 \leq 1, \ \|v\|_2 \leq 1\} = \sigma_{\max} (Z),\nonumber
\end{align}
where $\sigma_{\max}(Z)$ denotes the largest singular value of $Z$ (and thus $\| \cdot \|_*^{\circ}$ is the spectral norm).  In this case, given any first-order optimal point $(\tilde{U},\tilde{V},\tilde{Q})$, then Proposition \ref{prop:approx_bound} guarantees that the distance of the current point from the global minimum is bounded by $\lambda \|\hat{X}\|_* [\sigma_{\max}(\tfrac{-1}{\lambda} \nabla_X \ell(Y,\tilde{U}\tilde{V}^T,\tilde{Q}))-1]$.  If $(\tilde{U},\tilde{V},\tilde{Q})$ is a global minimizer, then the largest singular value term will be equal to 1 (and hence the bound is 0), while if the largest singular value term is greater than 1 this indicates that $(\tilde{U},\tilde{V})$ do not have sufficiently many columns to represent the global optimum (or some of the columns are redundant), and the size of the representation should be increased. Further, appending the largest singular vector pair $(u,v)$ to the factorization $U \leftarrow [\tilde{U} \ \tau u]$ and $V \leftarrow [\tilde{V} \ \tau v]$ (as this is the vector pair that achieves the supremum of the polar function) will be guaranteed to reduce the objective function for some step size $\tau > 0$. 

\subsection{Reachability of Global Minima via Local Descent}
\label{subsec:desc_strat}
Building on the above results, which provide sufficient conditions to guarantee global optimality, we now describe a generic meta-algorithm that is guaranteed to reach the global minimum of problem \eqref{eq:mat_fact_obj} via local descent, provided the number of columns in $(U,V)$ is allowed to be variable and adapted to the data via regularization.  In particular, recall that the condition of Theorem \ref{thm:mat_fact_0_min} that one column of $(U,V)$ must be entirely 0 is essentially a check that $r$ is sufficiently large to represent the optimal factorization of a global optimum, $\hat X$, of \eqref{general_obj_nofact}. If the condition is not satisfied, then the size of the factorization can be increased by appending a column of all zeros without changing the value of the objective and then proceeding with local descent.  This suggests a meta-algorithm of alternating between local descent on \eqref{eq:mat_fact_obj}, followed by incrementing $r$ by appending a column of all zeros.  Algorithm \ref{alg:mat_fact_meta} outlines the steps of this approach and adds an additional step to remove redundant columns of $(U,V)$ from the factorization (without changing the value of the objective function), and the following result formally proves that the meta-algorithm will find a global minimizer of \eqref{general_obj_nofact} and \eqref{eq:mat_fact_obj}, with the number of columns in $(U,V)$ never growing larger than the dimensionality of $X$.
\begin{theorem}
\label{thm:non_inc_path}
Given a function $\ell(Y,X,Q)$ that is jointly convex in $(X,Q)$ and once differentiable w.r.t. $X$; a rank-1 regularizer $\theta$ that satisfies the conditions in Definition \ref{def:mat_fact_reg}; and a constant $\lambda > 0$, then Algorithm \ref{alg:mat_fact_meta} will monotonically decrease the value of \eqref{eq:mat_fact_obj} and terminate at a global minimum of \eqref{general_obj_nofact} and \eqref{eq:mat_fact_obj} with $r \leq \max \{DN+1 ,r_{init} \}$.
\end{theorem}
\begin{proof} See Appendix \ref{sec:non_inc_path_proof}. \end{proof}

Additionally, note that the proof of Theorem \ref{thm:non_inc_path} provides a formal proof that the error surface of \eqref{eq:mat_fact_obj} does not contain any local minima which require one to increase the value of the objective function to escape from, provided the number of columns are initialized to be sufficiently large.
\begin{corollary}
Under the conditions in Theorem \ref{thm:non_inc_path}, if $r > DN$ then from any $(U_0,V_0,Q_0)$ such that $f(U_0,V_0,Q_0) < \infty$ there must exist a non-increasing path from $(U_0,V_0,Q_0)$ to a global minimizer of $f(U,V,Q)$.
\end{corollary}

Finally, we pause to caution that the above results apply to local descent, not necessarily first-order descent (e.g., gradient descent or the algorithm we describe below), and many first-order descent algorithms can typically only guarantee convergence to a critical point but not necessarily a local minimum.  As a result the main computational challenge 
is to find a local descent direction from a critical point, which can be accomplished by finding a $(u,v)$ pair 
such that $\theta(u,v) \leq 1$ and $u^T (\tfrac{-1}{\lambda} \nabla_X \ell(Y,UV^T,Q)) v > 1$.  However, to find such a pair (if it exists) in the worst case we would need to be able to solve the polar problem to guarantee global optimality, which can be challenging as discussed above, but note that from Proposition \ref{prop:approx_bound}, as it becomes harder to find $(u,v)$ pairs that can be used to decrease the objective function (i.e., the true value of the polar function at a critical point moves closer to 1) we are also guaranteed to be closer to the global minimum.

\begin{algorithm}
   \caption{\bf (Meta-Algorithm)}
   \label{alg:mat_fact_meta}
\begin{algorithmic}
	\INPUT{Initialization for variables, $(U_{init},V_{init},Q_{init})$}
	\WHILE{Not Converged}
		\STATE{Local descent to arrive at a local minimum $(\tilde{U},\tilde{V},\tilde{Q})$.}
		\IF{$(\tilde{U}_i,\tilde{V}_i)=(0,0)$ for some $i\in [r]$.}
			\STATE{At global optimum, return.}
		\ELSE
			\IF{$\exists \beta \in \Re^r / 0$ such that $\sum_{i=1}^r \beta_i \tilde U_i \tilde V_i^T = 0$.}
				\STATE{Scale $\beta$ so that $\min_i \beta_i = -1$.}
				\STATE{Set $(U_i,V_i) \leftarrow ((1+\beta_i)^{1/2} \tilde U_i,(1+\beta_i)^{1/2} \tilde V_i)$.}
				\STATE{$\backslash\backslash$ We now have an all-zero column in $(U,V)$.}
			\ELSE
				\STATE{Append an all zero column to $(U,V)$.}
				\STATE{$(U,V) \leftarrow ([\tilde U \ \ 0],[\tilde V \ \ 0])$.}
			\ENDIF
		\ENDIF
		\STATE{Continue loop.}
	\ENDWHILE
\end{algorithmic}
\end{algorithm}

\section{Minimization Algorithm}

In the previous section we described and analyzed a generic meta-algorithm that can be used 
to solve the family of problems we consider. However, while the meta-algorithm is useful from an analysis perspective, it still needs to be paired with a computationally practical algorithm to be of use.  Here we describe a specific and easily implementable algorithm that applies to a wide range of the problems discussed above that allows one to be guaranteed convergence to a critical point of the non-convex objective function \eqref{eq:mat_fact_obj}.  

Before we begin the discussion of the algorithm, note that in addition to the conditions included in Theorem \ref{thm:mat_fact_0_min}, the particular method we present here assumes that the gradients of the loss function $\ell(Y,UV^T,Q)$ w.r.t. $U$ and w.r.t. $V$ (denoted as $\nabla_U \ell(Y,UV^T,Q)$ and $\nabla_V \ell(Y,UV^T,Q)$, respectively) are Lipschitz continuous (i.e. the gradient w.r.t. $U$ is Lipschitz continuous for any fixed value of $V$ and vice versa). For a variable $Z$, $L_Z^k$ will be used to notate the Lipschitz constant of the gradient w.r.t. $Z$ with the other variables held fixed at iteration $k$. Additionally, we will assume that $\theta(u,v)$ is convex in $u$ if $v$ is held fixed and vice versa (but not necessarily jointly convex in $u$ and $v$).  Under these assumptions on $\ell$ and $\theta$, the bilinear structure of our objective function \eqref{eq:mat_fact_obj} gives convex subproblems if we update $U$ or $V$ independently while holding the other fixed, making an alternating minimization strategy efficient
and easy to implement. Further, we assume that $\ell(Y,UV^T,Q) = \hat{\ell}(Y,UV^T,Q) + H(Q)$ where $\hat{\ell}(Y,UV^T,Q)$ is a convex, once differentiable function of $Q$ with Lipschitz continuous gradient 
and $H(Q)$ is convex but possibly non-differentiable.\footnote{Note that the assumption that there is a component of the objective that is differentiable w.r.t. $Q$, $\hat{\ell}$, is only needed to use the particular update strategy we describe here.  In general one could also optimize objective functions that are non-differentiable w.r.t. $Q$ (but which do need to be convex w.r.t. $Q$) by doing a full minimization w.r.t. $Q$ at each iteration instead of just a proximal gradient update.  See \cite{Xu:SIAM2013} for details.}

The updates to our variables are made using accelerated proximal-linear steps similar to the FISTA algorithm, which entails solving a proximal operator of an extrapolated gradient step to update each variable \cite{Beck2009,Xu:SIAM2013}.  The general structure of the alternating updates we use is given in Algorithm \ref{algorithm_low_rank_fact}, and the key point is that to update either $U$, $V$, or $Q$ the primary computational burden lies in calculating the gradient of the loss function and then calculating a proximal operator.  The structure
of the non-differentiable term in \eqref{eq:mat_fact_obj} allows the proximal operators for $U$ and $V$ to be separated into columns, greatly reducing the complexity of calculating the proximal operator and offering the potential for substantial parallelization.  

{\fontsize{10}{1}
\begin{algorithm}
	\caption{\bf (Structured Matrix Factorization)}
	\label{algorithm_low_rank_fact}
\begin{algorithmic}
   \STATE {\bfseries Input:} $Y$, $U^0$, $V^0$, $Q^0$, $\lambda$, NumIter
	 \STATE{ Initialize $\hat{U}^1=U^0$, $\hat{V}^1=V^0$, $\hat{Q}^1=Q^0$, $t^0=1$}
	 \FOR{$k=1$ {\bfseries to} NumIter}
		 \STATE{ \textbackslash \textbackslash Calculate gradient of loss function w.r.t. $U$}
		 \STATE{ \textbackslash \textbackslash evaluated at the extrapolated point $\hat{U}$}
	   \STATE{ $G_U^k = \nabla_U \ell(Y,\hat{U}^k (V^{k-1})^T,Q^{k-1})$}
		  \STATE{ $P = \hat{U}^k-G_U^k / L_U^k$}
		 \STATE{ \textbackslash \textbackslash Calculate proximal operator of $\theta$}
		 \STATE{ \textbackslash \textbackslash for every column of $U$}
		 \FOR{$i=1$ {\bfseries to} number of columns in $U$}
		   \STATE{$U_i^k = \textbf{prox}_{\lambda \theta(\cdot,V_i^{k-1}) / L_U^k}(P_i)$}
		 \ENDFOR
		 \STATE{ \textbackslash \textbackslash Repeat similar process for $V$}
		 \STATE{$G_V^k = \nabla_V \ell(Y,U^k (\hat{V}^k)^T,Q^{k-1})$}
		 \STATE{$W = \hat{V}^k-G_V^k / L_V^k$}
		 \FOR{$i=1$ {\bfseries to} number of columns in $V$}
		   \STATE{$V_i^k = \textbf{prox}_{\lambda \theta(U_i^k,\cdot) / L_V^k }(W_i)$}
		 \ENDFOR
		\STATE{ \textbackslash \textbackslash Repeat again for $Q$}
		 \STATE{$G_Q^k = \nabla_Q \hat{\ell}(Y,U^k (V^k)^T,\hat{Q}^k)$}
		  \STATE{$R = \hat{Q}^k-G_Q^k / L_Q^k$}
		 \STATE{$Q^k = \textbf{prox}_{H (Y,U^k (V^k)^T,\cdot) / L_Q^k }(R)$}
		 
		\STATE{ \textbackslash \textbackslash Update extrapolation based on prior iterates}
		\STATE{ \textbackslash \textbackslash Check if objective decreased}
		\IF{ $obj(U^k,V^k,Q^k) < obj(U^{k-1},V^{k-1},Q^{k-1})$ }
			\STATE{ \textbackslash \textbackslash The objective decreased, update extrapolation}
			\STATE{$t^k = (1+\sqrt{1+4(t^{k-1})^2})/2$}
			\STATE{$\mu = (t^{k-1}-1)/2$}
			\STATE{$\mu_U = \min\{\mu,\sqrt{L_U^{k-1} / L_U^k}\}$ }
			\STATE{$\mu_V = \min\{\mu,\sqrt{L_V^{k-1} / L_V^k}\}$ }
			\STATE{$\mu_Q = \min\{\mu,\sqrt{L_Q^{k-1} / L_Q^k}\}$ }
			\STATE{$\hat{U}^{k+1} = U^k + \mu_U(U^k-U^{k-1})$}
			\STATE{$\hat{V}^{k+1} = V^k + \mu_V(V^k-V^{k-1})$}
			\STATE{$\hat{Q}^{k+1} = Q^k + \mu_Q(Q^k-Q^{k-1})$}
		\ELSE
			\STATE{ \textbackslash \textbackslash The objective didn't decrease.}
			\STATE{ \textbackslash \textbackslash Run again without extrapolation.}
			\STATE{$t^k = t^{k-1}$}
			\STATE{$\hat{U}^{k+1} = U^{k-1}$}
			\STATE{$\hat{V}^{k+1} = V^{k-1}$}
			\STATE{$\hat{Q}^{k+1} = Q^{k-1}$}
		\ENDIF
   \ENDFOR
\end{algorithmic}
\end{algorithm}
}

\subsection{Proximal Operators of Structured Factors}

Recall from the introductory discussion that one means to induce general structure in the factorized matrices is to regularize the columns of a factorized matrix with an $l_2$ norm, to limit the rank of the solution, plus a general gauge function, to induce specific structure in the factors.  For example, potential forms of the rank-1 regularizers could be of the form $\theta(u,v) = \|u\|_2 \|v\|_2 + \gamma \sigma_u(u) \sigma_v(v)$ or $\theta(u,v) = (\|u\|_2 + \gamma_u \sigma_u(u))(\|v\|_2 + \gamma_v \sigma_v(v))$, where the $\sigma_u$ and $\sigma_v$ gauge functions are chosen to encourage specific properties in $U$ and $V$, respectively. In this case, to apply Algorithm \ref{algorithm_low_rank_fact} we need a way to solve the proximal operator of the $l_2$ norm plus a general gauge function.  While the proximal operator of the $l_2$ norm is simple to calculate, even if the proximal operator of the gauge function is known, in general the proximal operator of the sum of two functions is not necessarily easy to compute or related to the proximal operators of the individual functions. Fortunately, the following result shows that for the sum of the $l_2$ norm plus a general gauge function, the proximal operator can be solved by sequentially calculating the two proximal operators.

\begin{proposition}
\label{prop:l2prox}
Let $\sigma_C$ be any gauge function. The proximal operator of $\psi(x) = \lambda \sigma_C(x) + \lambda_2 \| x \|_2$ is the composition of the proximal operator of the $l_2$ norm and the proximal operator of $\sigma_C$, i.e., $\prox_\psi(y) = \prox_{\lambda_2 \| \cdot \|_2}(\prox_{\lambda \sigma_C}(y))$.
\end{proposition}
\begin{proof} See Appendix \ref{sec:l2prox_proof}. \end{proof}

Combining these results with Theorem \ref{thm:mat_fact_0_min} and our previously discussed points, we now have a potential strategy to search for structured low-rank matrix factorizations as we can guarantee global optimality if we can find a local minimum with an all-zero column in $U$ and $V$ (or equivalently a local minimum that also satisfies condition 3 of Corollary \ref{cor:global_min_conds}), and the above proposition provides a means to efficiently solve proximal operator problems that one typically encounters in structured factorization formulations.  However, recall the critical caveats to note about the optimization problem; namely that first order descent methods as we have presented here are only guaranteed to converge to a critical point, not necessarily a local minimum, and solving the polar problem to guarantee global optimality can be quite challenging.

%% file: PAMI16_StructMatFact_applications.tex
\section{Applications in Image Processing}
In the next two sections, we will explore applications of the proposed structured matrix factorization framework to two image processing problems: spatiotemporal segmentation of neural calcium imaging data and hyperspectral compressed recovery. Such problems are well modeled by low-rank linear models with square loss functions under the assumption that the spatial component of the data has low total variation (TV) and is optionally sparse in the row and/or column space. 
Specifically, we will consider the following objective
\begin{align}
\label{app_obj}
\min_{U,V,Q} \frac{1}{2} \| Y-&\mathcal{A}(UV^T) - \mathcal{B}(Q) \|_F^2 + \lambda \sum_i \|U_i\|_u \|V_i\|_v
 \\ \nonumber &\textnormal{(optionally s.t.)} \ U \geq 0, V \geq 0
\end{align}
where $\mathcal{A}(\cdot)$ and $\mathcal{B}(\cdot)$ are linear operators, and the $\|\cdot\|_u$ and $\|\cdot\|_v$ norms have the form given by
\begin{align}
\label{U_norm}
\| \cdot \|_u &= \nu_{u_1} \| \cdot \|_1 + \nu_{u_{TV}} \| \cdot\|_{TV} + \nu_{u_2}\| \cdot \|_2 \\
\label{V_norm}
\| \cdot \|_v &= \nu_{v_1} \| \cdot \|_1 + \nu_{v_{TV}} \| \cdot \|_{TV} + \nu_{v_2} \| \cdot \|_2,
\end{align}
for non-negative scalars $\nu$.  Here, the $\ell_1$ and TV terms allow the incorporation of sparsity or spatial coherence, respectively, while the $\ell_2$ term limits the rank of the solution though the connection with the variational form of the nuclear norm.  Recall that the anisotropic TV of $x$ is defined as \cite{Birkholz:JCAM11}
\begin{equation}
\|x\|_{TV} \equiv \sum_i \sum_{j \in \textit{N}_i} \left|x_i - x_j \right|,
\end{equation}
where $\textit{N}_i$ is the set of pixels in the neighborhood of pixel $i$.

Further, note that this objective function exactly fits within the proposed framework as we can define a rank-1 regularizer $\theta(u,v) = \|u\|_u \|v\|_v$ and optionally add indicator functions on $u$ and/or $v$ to enforce non-negativity constraints.  Additionally, the proximal operators for the proposed regularization function can be easily evaluated as we show in Appendix \ref{sec:L1TV_prox}.

\section{Neural Calcium Imaging Segmentation}
This section demonstrates the applicability of the proposed structured matrix factorization framework to the problem of segmenting calcium imaging data. Calcium imaging is a rapidly growing microscopy technique in neuroscience that records fluorescent images from neurons that have been loaded with either synthetic or genetically encoded fluorescent calcium indicator molecules. When a neuron fires an electrical action potential (or spike), calcium enters the cell and binds to the fluorescent calcium indicator molecules, changing the fluorescence properties of the molecule. By recording movies of the calcium-induced fluorescent dynamics it is possible to infer the spiking activity from large populations of neurons with single neuron resolution \cite{Stosiek:PNAS2003}. 

Specifically, let $y \in \Re^t$ be the fluorescence time series from a single neuron (normalized by the baseline fluorescence) during $t$ imaging frames. We can infer the neuron's spiking activity $x \in \Re^t$ (each entry of $x$ is monotonically related to the number of action potentials of the neuron during that imaging frame) by solving a Lasso like problem:
\begin{equation}
\label{jovo_lasso}
\hat{x} = \argmin_{x \geq 0} \frac{1}{2}\left\|y - Dx \right\|_2^2 + \lambda \left\| x \right\|_1,
\end{equation} 
where $D \in \Re^{t \times t}$ is a matrix that applies a convolution with a known decaying exponential to model the change in fluorescence resulting from a neural action potential \cite{jovo:foopsi}. However, the above model only applies to a single fluorescence time series extracted from a manually segmented data volume, and a significant challenge in neural calcium imaging is that the data can have a significant noise level, making manual segmentation extremely challenging.  Additionally, it is also 
it is possible for two neurons to overlap in the spatial domain if the focal plane of the microscope is thicker than the size of the distinct neural structures in the data, making simultaneous spatiotemporal segmentation necessary. A possible strategy to address these issues would be to extend \eqref{jovo_lasso} to estimate spiking activity for the whole data volume via the objective
\begin{equation}
\label{jovo_lasso_volume}
\hat{X} = \argmin_{X \geq 0} \frac{1}{2}\left\|Y - DX \right\|_F^2 + \lambda \left\| X \right\|_1,
\end{equation}
where now each column of $Y \in \Re^{t \times p}$ contains the fluorescent time series for a single pixel and the 
corresponding column of $\hat{X} \in \Re^{t \times p}$ contains the estimated spiking activity for that pixel.  However, due to the significant noise often present in the actual data, and the fact that each column of $\hat{X}$ is estimated with data from a single pixel in the dataset without the benefit of averaging over a spatial area, solving \eqref{jovo_lasso_volume} directly typically gives poor results. To address this issue, \citet{jovo_lowrank} exploits the knowledge that if two pixels are from the same neural structure they should have identical spiking activities. Therefore, the matrix $X$ can be well approximated as $X \approx UV^T = \sum_{i=1}^r U_i V_i^T$, where $U_i\in\Re^t$ is the spiking activity of one neural structure, $V_i\in\Re^p$ is a shape image of the same neural structure, and $r\ll p$ is the number of neural structures in the data. This suggests adding a low-rank regularization term to \eqref{jovo_lasso_volume}, which leads to \citet{jovo_lowrank}
\begin{equation}
\label{jovo_lasso_lowrank}
\hat{X} = \argmin_{X \geq 0} \frac{1}{2}\left\|Y - DX \right\|_F^2 + \lambda \left\| X \right\|_1 + \lambda_2 \| X \|_*.
\end{equation}
Given $\hat X$, one can estimate the temporal and spatial features $(U,V)$ by applying non-negative matrix factorization to $\hat{X}$.

While \eqref{jovo_lasso_lowrank} provides a nice model of spiking activity within a dataset, recall from the introduction that solving a factorization problem in the product space (i.e., solving for $X$) is somewhat unsatisfactory as it does not provide us with the desired factors. For example, the number of calcium transients from each neuron should be small, hence each column of $U$ should be sparse. Likewise, each neuron should occupy a small fraction of the image, hence each column of $V$ should also be sparse. Notably, it can be shown that problem \eqref{app_obj} is equivalent to a standard Lasso estimation when both 
$U$ and $V$ are regularized by the $l_1$ norm \cite{bach08}, while combined $l_1$, $l_2$ norms of the form \eqref{U_norm} and \eqref{V_norm} with $\nu_{u_{TV}}=0$ promote solutions that are simultaneously sparse and low rank.  Thus, the projective tensor norm \eqref{eq:tensor_norm} can generalize the two prior methods for calcium image processing by providing regularizations that are sparse or simultaneously sparse and low-rank, while also having the advantage of solving for the desired factors directly.  Further, by working in the factorized space we can also model additional known structure in the factors.  In particular, we extend formulations \eqref{jovo_lasso_volume} and \eqref{jovo_lasso_lowrank} by noting that neighboring pixels are likely to be from the same neural structure and thus have identical spiking activity, implying low TV in the spatial domain. We demonstrate the flexible nature of our formulation \eqref{app_obj} by using it to process calcium image data with regularizations that are either sparse, simultaneously sparse and low-rank, or simultaneously sparse, low-rank, and with low TV.  Additionally, by optimizing \eqref{app_obj} to simultaneously estimate temporal spiking activity $U$ and neuron shape $V$, with $\mathcal{A}(UV^T)=DUV^T$, we inherently find spatial and temporal features in the data (which are largely non-negative even though we do not explicitly constrain them to be) directly from our optimization without the need for an additional matrix factorization step.  Finally, note that the $\mathcal{B}(Q)$ term can be used to fit the background intensity of the pixels by taking $\mathcal{B}(Q) = \textbf{1} Q^T$ for a vector $Q \in \Re^p$, and if the data exhibits temporal variations in pixel intensities not due to calcium activity, such as from slow movements of the sample or photo-bleaching, this can also be modeled via an appropriate choice of the $\mathcal{B}$ operator.  For the experiments presented here the data has been normalized by background intensity, so the $\mathcal{B}(Q)$ term is not used.

\subsection{Simulation Data}

We first tested our algorithm on a simulated phantom dataset which was constructed with 19 non-overlapping spatial regions (see Figure \ref{fig:phantom_labels}, left panel) and 5 randomly timed action potentials and corresponding calcium dynamics per region.  The phantom was 200 frames of 120x125 images, and the decaying exponentials in $D$ had a time constant of $1.33\bar{3} \, sec$ with a simulated sampling rate of $10 \, Hz$.  Gaussian white noise was added to the modeled calcium signal to produce an SNR of approximately \mbox{-16dB}.

Using this phantom, we used Algorithm \ref{algorithm_low_rank_fact} to solve the formulation given in \eqref{app_obj} with different $\nu$ parameters for the norms in \eqref{U_norm} and \eqref{V_norm}.  In particular, we used just sparse and low-rank regularization by taking $[\nu_{u_1},\nu_{u_{TV}},\nu_{u_2}] = [\nu_{v_1},\nu_{v_{TV}},\nu_{v_2}] = [1,0,1]$ and $\lambda = 1.5 \sigma$, where $\sigma$ denotes the standard deviation of the Gaussian noise.  Then, to demonstrate the benefit of adding TV regularization in the spatial domain, we used an 8-connected lattice for the TV graph,\footnote{The regularization parameters were roughly tuned by hand to produce the best qualitative results for the two experimental conditions (i.e., sparse + low-rank w/wo TV)} 
$[\nu_{u_1},\nu_{u_{TV}},\nu_{u_2}] = [1,0,1]$, $[\nu_{v_1},\nu_{v_{TV}},\nu_{v_2}] = [1,1,1]$ and $\lambda = 0.4 \sigma$.  For the sparse + low-rank condition $U$ was initialized to be the identity matrix.  For the experiments that include TV regularization, we again conducted experiments with $U$ initialized to be the identity, and to study the effects of different initializations, we additionally also performed experiments with $U$ initialized with 50 columns, where each entry in $U$ was initialized to a random value uniformly distributed in  $[0,1]$ (in all cases $V$ was initialized as 0).

\begin{figure}
\centering
\includegraphics[width=\linewidth]{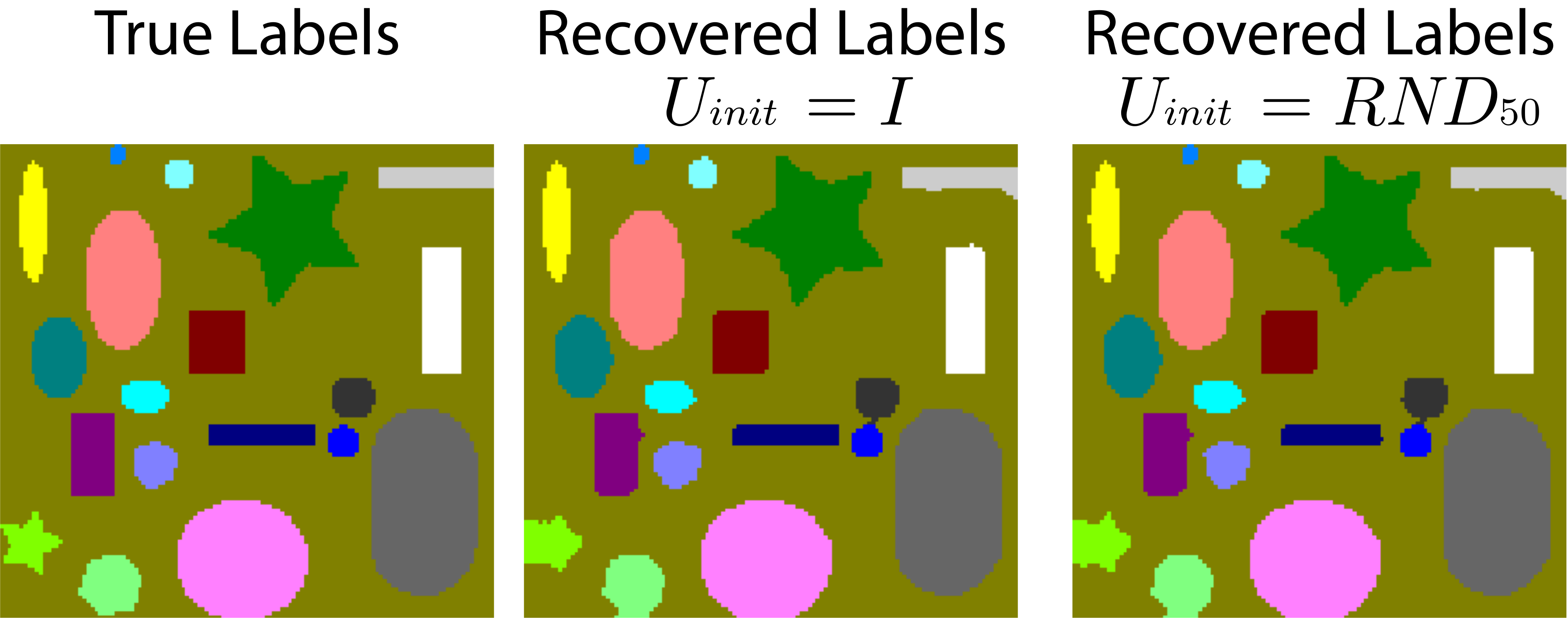}
\vspace{-6mm}
\caption[Recovered spatial segmentation from phantom dataset.] {Recovered spatial segmentations from phantom dataset.  \textit{Left}: True spatial labels. \textit{Middle}: Spatial labels recovered with sparse + low-rank + TV regularization, with $U$ initialized as an identity matrix. \textit{Right}: Same as the middle panel but with $U$ initialized as 50 columns of random values uniformly distributed between $[0,1]$.}
\label{fig:phantom_labels}
\vspace{-2mm}
\end{figure}

\begin{figure}
\centering
\includegraphics[width=\linewidth]{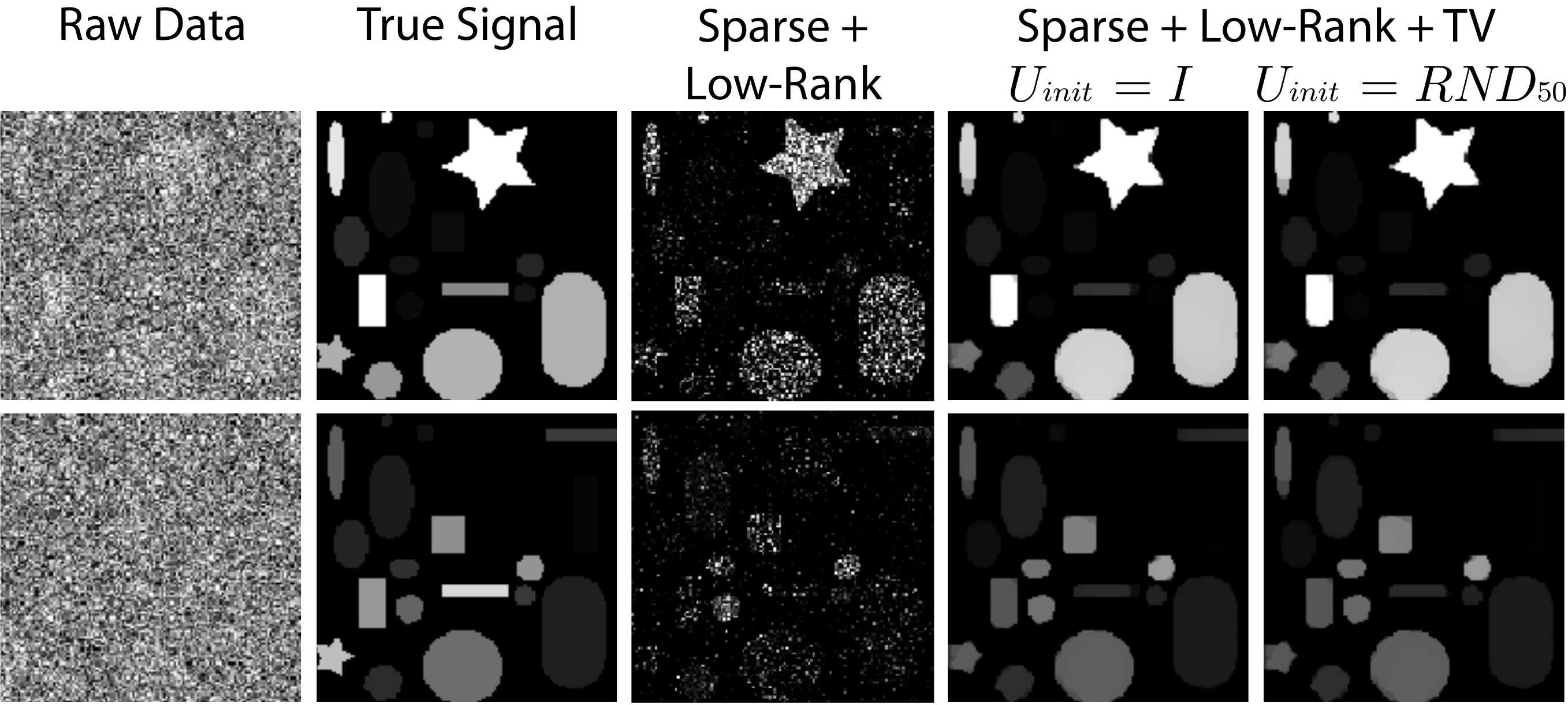}
\vspace{-6mm}
\caption[Example reconstructed calcium signal from phantom dataset.] {Example reconstructed calcium signal from phantom dataset.  The two rows correspond to two different example image frames.  \textit{From left to right}: Raw data. True calcium signal. Reconstruction with sparse + low-rank regularization.  Reconstruction with sparse + low-rank + TV regularization with $U$ initialized as the identity. Reconstruction with sparse + low-rank + TV regularization with $U$ initialized as 50 columns of random values uniformly distributed in $[0,1]$.} 
\label{fig:phantom_frames}
\vspace{-2mm}
\end{figure}

\begin{figure}
\centering
\includegraphics[width=\linewidth]{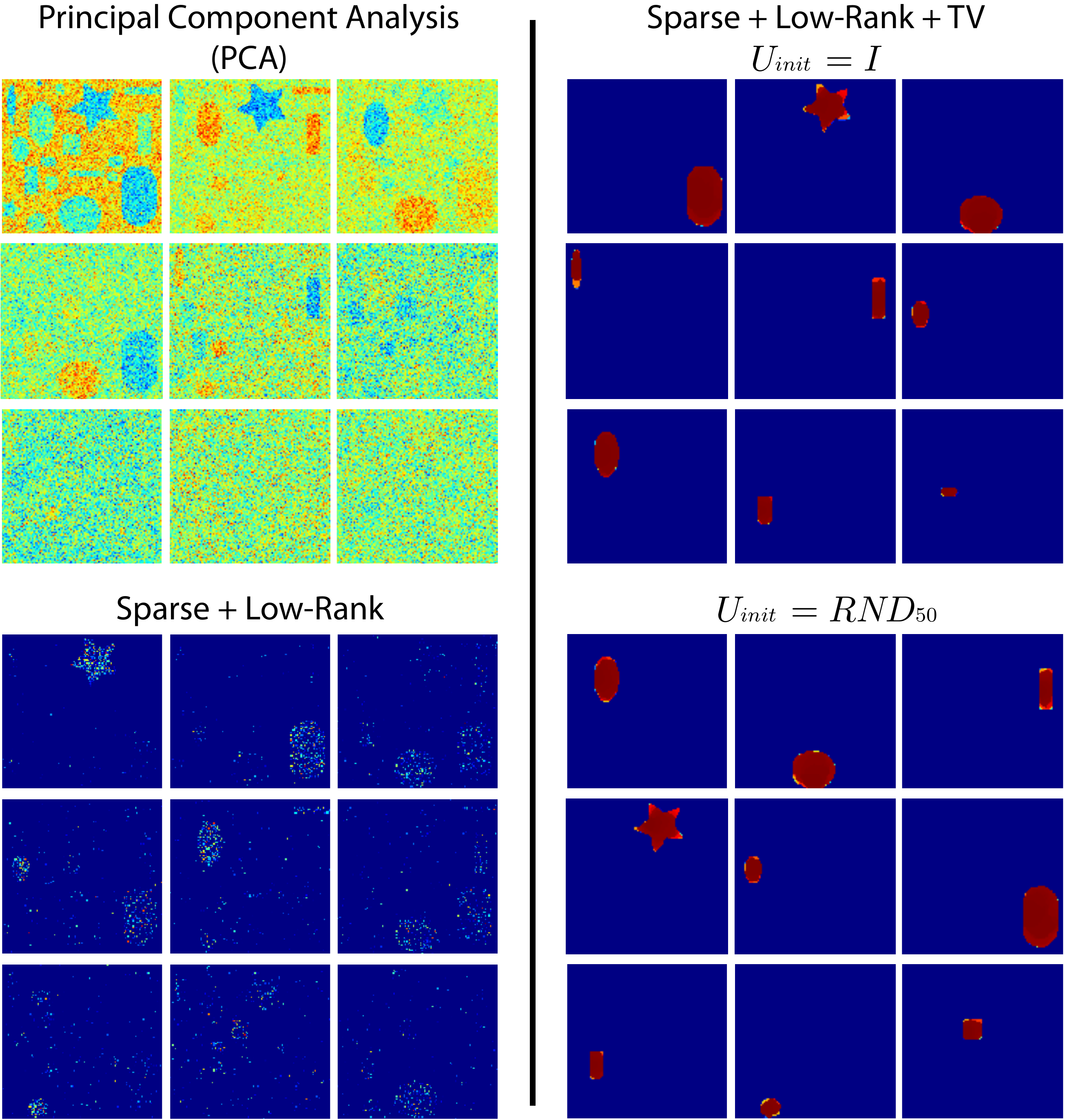}
\vspace{-6mm}
\caption[Example recovered spatial components from phantom dataset.] {Example recovered spatial components from phantom dataset.  \textit{Top Left}: First 9 most significant spatial components recovered via Principal Component Analysis (PCA). \textit{Bottom Left}: First 9 most significant spatial components recovered with sparse and low-rank regularization. \textit{Top Right}: First 9 most significant spatial components recovered using sparse, low-rank, and TV regularization, with $U$ initialized as the identity. \textit{Bottom Right}: Same as the top right panel but with $U$ initialized as 50 columns of random values uniformly distributed in $[0,1]$.}
\label{fig:phantom_regions}
\end{figure}

Figure \ref{fig:phantom_frames} shows two example reconstructions of the calcium signal estimated with our algorithm with different regularization conditions. Figure \ref{fig:phantom_regions} shows example spatial components recovered by our algorithm as well as spatial components recovered by PCA.  For each case, the components shown are the first 9 most significant components (i.e., those with the largest value of $\|U_i\|_u \|V_i\|_v$).\footnote{\label{foot:duplicates} Note that the differences in the specific components shown in Figure \ref{fig:phantom_regions} between the two initializations of $U$ is due to the fact that the structure of the objective function \eqref{app_obj} allows for components to be duplicated without changing the value of the objective function.  For example, if $U = [U_1 \ U_2]$ and $V = [V_1 \ V_2]$, then $\tilde{U} = [U_1 \ 0.2 U_2 \ 0.8 U_2]$ and $\tilde{V} = [V_1 \ V_2 \ V_2]$ will give identical objective function values.}  Note that although we only show the first 9 spatial components here for compactness, the remaining components also closely correspond to the true spatial regions and allow for the true spatial segmentation to be recovered.

The recovered temporal components for the 9 regions shown in Figure \ref{fig:phantom_regions} are plotted in Figure \ref{fig:temporal_plots} along with the corresponding true temporal spike times (red dots) for the sparse + low-rank + total-variation (SRLTV) regularization conditions. 
The spatial segmentation obtained via SLRTV regularization with the two different initializations for $U$ is shown in Figure \ref{fig:phantom_labels}. This segmentation was generated by simply finding connected components of the non-zero support of the spatial components, then any two connected components that overlapped by more than 10\% were merged (note that this step is largely only necessary to combine duplicate components --see footnote \ref{foot:duplicates}-- and the results are very insensitive to the choice of the percentage of overlap as any duplicated components had almost identical non-zero supports).  Despite the very high noise level, adding 
SLRTV regularization recovers the true spatial and temporal components with very high accuracy and faithfully reconstructs the true calcium signal.  Further, this performance is robust to the choice of initialization, as initializing  $U$ with either the identity or random values still faithfully recovers true spatial and temporal components.  Additionally, despite the very different initializations, the relative error between the two final objective values (given as $|obj_1 - obj_2| / \min \{ obj_1,obj_2 \}$, where $obj_1$ and $obj_2$ denote the final objective values for the 2 different initializations) was only \mbox{$3.8833 \times 10^{-5}$}.

\begin{figure}
\centering
\includegraphics[width=\linewidth]{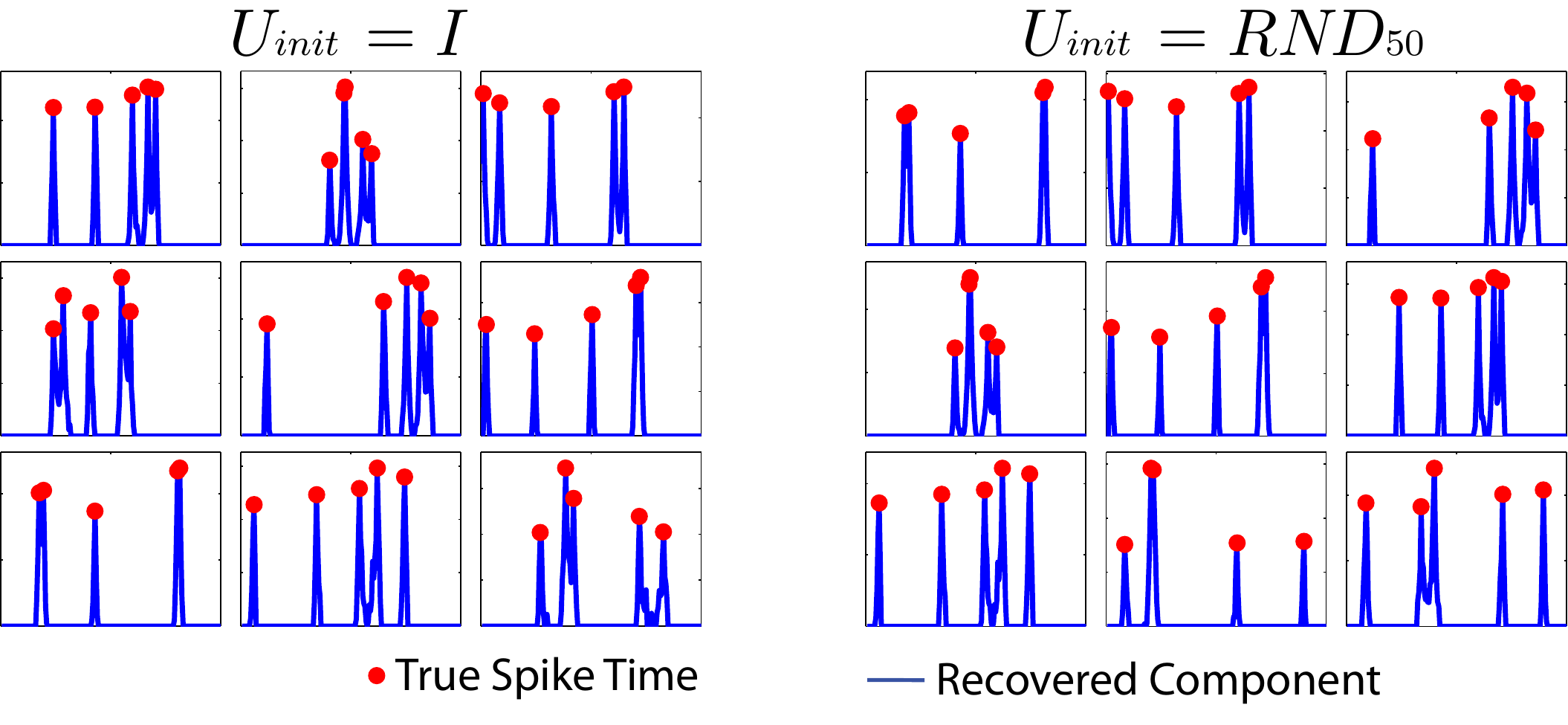}
\vspace{-6mm}
\caption[Reconstructed temporal components from phantom dataset.] {Reconstructed spike trains from phantom dataset with sparse + low-rank + TV for the components shown in Figure \ref{fig:phantom_regions}.  Blue lines are the temporal components estimated  by our algorithm, while the red dots correspond to the true temporal spike times. \textit{Left Panel}: Reconstruction with $U$ initialized as the identity. \textit{Right Panel}: Reconstruction with $U$ initialized as 50 columns of random values uniformly distributed in $[0,1]$.} 
\label{fig:temporal_plots}
\vspace{-2mm}
\end{figure}

\begin{figure*}
\centering
\includegraphics[clip=true,trim=0 239 0 0,scale=0.85]{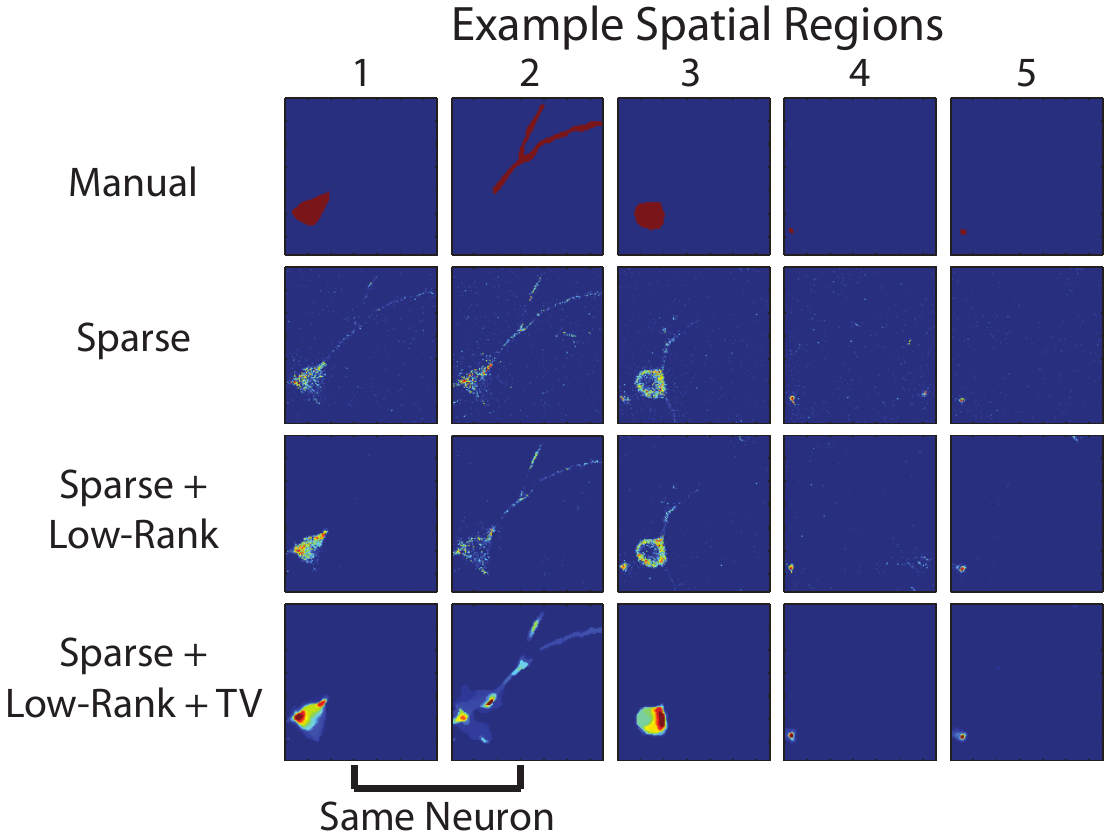}
\hfill
\raisebox{3.7mm}{\includegraphics[clip=true,trim=0 0 47 251,scale=0.85]{NeuralFull2.pdf}}
\caption[Results from an \textit{in vivo} calcium imaging dataset.] {Results from an \textit{in vivo} calcium imaging dataset. \textit{Left}: Spatial features for 5 example regions. 
(Top Row) Manually segmented regions.  (Bottom 3 Rows) Corresponding spatial feature recovered by our method with various regularizations.
Note that regions 1 and 2 are different parts of the same neurons - see discussion in the text. \textit{Right}: Example frames from the
dataset corresponding to time points where the example regions display a significant calcium signal. (Top Row) Actual Data. (Bottom 3
Rows) Estimated signal for the example frame with various regularizations.}
\label{NeuralFull}
\end{figure*}

\begin{figure*}
\includegraphics[clip=true,trim=15 130 0 0,height=15.8mm]{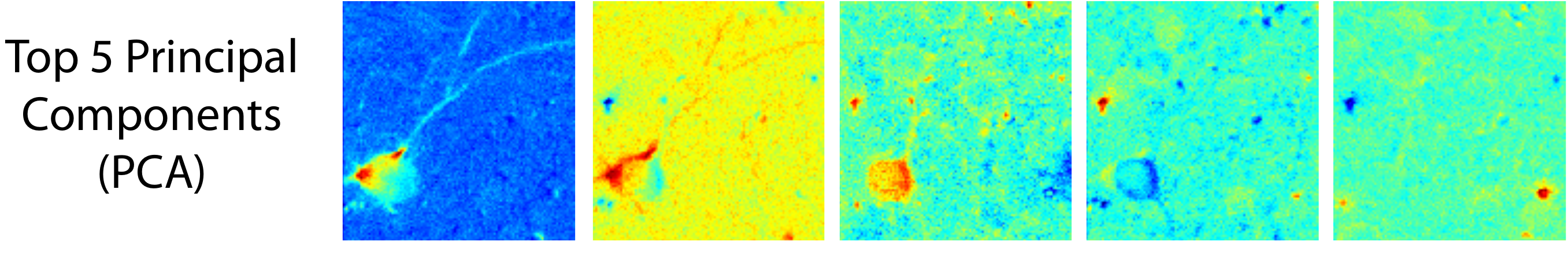}
\hfill
\includegraphics[clip=true,trim=0 0 124 130,height=15.8mm]{InVivoPCA_full.pdf}
\caption[Results of PCA applied to an \textit{in vivo} calcium imaging dataset.] {Results of PCA applied to an \textit{in vivo} calcium imaging dataset. \textit{Left}: The first 5 most significant spatial components from PCA analysis.  \textit{Right}: Example image frames reconstructed from the first 20 most significant Principal Components.  The example frames are the same is in Figure \ref{NeuralFull}.} 
\label{fig:invivo_pca}
\vspace{-2mm}
\end{figure*}

\subsection{\textit{In vivo} Calcium Image Data} 
We also tested our algorithm on actual calcium image data taken \textit{in vivo} from the primary auditory cortex of a mouse that was transfected with the genetic calcium indicator GCaMP5 \cite{Akerboom:JNeuro2012}. The left panel of Figure \ref{NeuralFull} shows 5 manually labeled regions from the dataset (top row) and the corresponding spatial features recovered by our algorithm (bottom 3 rows) under the various regularization conditions. The right panel of Figure \ref{NeuralFull} displays a frame from the dataset taken at a time point when the corresponding region had a significant calcium signal, with the actual data shown in the top row and the corresponding reconstructed calcium signal for that time point under the various regularization conditions shown in the bottom 3 rows.  We note that regions 1 and 2 correspond to the cell body and a dendritic branch of the same neuron. The manual labeling was purposefully split into two regions due to the fact that dendrites can have significantly different calcium dynamics from the cell body and thus it is often appropriate to treat calcium signals from dendrites as separate features from the cell body \cite{Spruston:NatRevNeuro2008}. 

The data shown in Figure \ref{NeuralFull} are particularly challenging to segment as the two large cell bodies (regions 1 and 3) are largely overlapping in space, necessitating a spatiotemporal segmentation. In addition to the overlapping cell bodies there are various small dendritic processes radiating perpendicular to (regions 4 and 5) and across (region 2) the focal plane that lie in close proximity to each other and have significant calcium transients. Additionally, at one point during the dataset the animal moves, generating a large artifact in the data. 
Nevertheless, by optimizing \eqref{app_obj} under the various regularization conditions, we observe that, as expected, the spatial features recovered by sparse regularization alone are highly noisy (Fig. \ref{NeuralFull}, row 2). Adding low-rank regularization improves the recovered spatial features, but the features are still highly pixelated and contain numerous pixels outside of the desired regions (Fig. \ref{NeuralFull}, row 3). Finally, by incorporating TV regularization, our method produces coherent spatial features which are highly similar to the desired manual labellings (Fig. \ref{NeuralFull}, rows 1 and 4), noting again that these features are found directly from the alternating minimization of \eqref{app_obj} without the need to solve a secondary matrix factorization.  For comparison purposes, the top 5 spatial components recovered via PCA along with example image frames reconstructed using the top 20 principal components are shown in Figure \ref{fig:invivo_pca}.  Note that while the PCA spatial components have a rough correspondence to the neural structures in the data, a significant amount of post-processing would be required to recover the segmentation of a specific neural structure from the PCA representation.  Likewise, the example image frames recovered via PCA still contain a very large amount of noise. 

In the application of our structured matrix factorization algorithm to the \textit{in vivo} dataset, $U$ was initialized to be 100 uniformly sampled columns from a $599\times 599$ identity matrix, and $V$ was initialized as $V=0$, demonstrating the potential to reduce the problem size and achieve good results despite a very trivial initialization.  Similar to the phantom experiments, choosing $U$ to be initialized as random variables in $[0,1]$ produced nearly identical results (not shown). The regularization parameters were tuned manually to produce good qualitative performance for each regularization condition, and the specific values of the parameters are given in Table \ref{table:reg_cond}.

\begin{table}
\caption[Regularization parameters for \textit{in vivo} calcium imaging experiments.] {Regularization parameters for \textit{in vivo} calcium imaging experiments. $\sigma$ denotes the standard deviation of all of the voxels in the data matrix, $Y$.}
\label{table:reg_cond}
\vspace{-1mm}
\centering
\small
\begin{tabular}{@{}c@{\,} | @{\,}c@{\,}c@{\,}c@{}}
\hline
 & $\lambda$ & $[\nu_{u_1},\nu_{u_{TV}},\nu_{u_2}]$ & $[\nu_{v_1},\nu_{v_{TV}},\nu_{v_2}]$ \\
\hline
Sparse & $2 \sigma$ & $[1,0,0]$ & $[1,0,0]$ \\
Sparse + Low-Rank & $1.75 \sigma$ & $[1,0,1]$ & $[1,0,1]$ \\
Sparse + Low-Rank + TV & $0.5 \sigma$ & $[1,0,2.5]$ & $[1,0.5,1]$ \\
\hline
\end{tabular}
\vspace{-5mm}
\end{table}

We conclude by noting that while TV regularization can improve performance for a segmentation task, it also can cause a dilative effect when reconstructing the estimated calcium signal (e.g., distorting the size of the thin dendritic processes in the left two columns of the example frames in Figure \ref{NeuralFull}).  As a result, in a denoising task it might instead be desirable to only impose sparse and low-rank regularization.  The fact that we can easily and efficiently adapt our model to account for many different features of the data depending on the desired task highlights the flexible nature and unifying framework of our proposed formulation \eqref{app_obj}.
\section{Hyperspectral Compressed Recovery}

The second application we considered is recovering a hyperspectral image volume from a set of compressed measurements.  Hyperspectral imaging (HSI) is similar to regular digital photography, but instead of recording the intensities of light at just 3 wavelengths (red, green, blue) as in a typical camera, HSI records images for a very large number of wavelengths (typically hundreds or more). 

Due to the nature of hyperspectral imaging the data often display a low-rank structure.  For example, consider hyperspectral images taken during aerial reconnaissance.  If one was given the spectral signatures of various materials in the hyperspectral image volume (trees, roads, buildings, dirt, etc.), as well as the spatial distributions of those materials, then one could construct a matrix $U\in \Re^{t \times r}$, whose $i$\ts{th} column, $U_i$, contains the spectral signature of the $i$\ts{th} material (recorded at $t$ wavelengths) along with a matrix $V \in \Re^{p \times r}$ which contains the spatial distribution of the $i$\ts{th} material in its $i$\ts{th} column $V_i$ (where $p$ denotes the number of pixels in the image).  Then, $r$ corresponds to the number of materials present in the given HSI volume, and since typically $r \ll \min \{ t,p \}$ the overall HSI volume can be closely approximated by the low-rank factorization $Y \approx UV^T$. This fact, combined with the large data sizes typically encountered in HSI applications, has led to a large interest in developing compressed sampling and recovery techniques to compactly collect and reconstruct HSI datasets.  

In addition, an HSI volume also displays significant structure in the spatial domain because neighboring pixels are highly likely to correspond to the same material \cite{zhang13}. This combination of low-rank structure along with strong correlation between neighboring pixels in the spatial domain of an HSI dataset led the authors of \citet{golbabaee12} to propose a combined nuclear norm and TV regularization (NucTV) method to reconstruct HSI volumes from compressed measurements with the form
\begin{align}
\label{hyperspec_form}
\!\!\!
\min_X \| X \|_* \!+\! \lambda \! \sum_{i=1}^t \! \|(X^i)^T\|_{TV} \st \|Y \!-\! \mathcal{A} (X) \|_F^2 \leq \epsilon.\!\!
\end{align}
Here $X \in \Re^{t \times p}$ is the desired HSI reconstruction with $t$ spectral bands and $p$ pixels, $X^i$ denotes the $i$\ts{th} row of $X$ (or the $i$th spectral band), $Y \in \Re^{t \times  m}$ contains the observed samples (compressed at a subsampling ratio of $m/p$), and $\mathcal{A}(\cdot)$ denotes the compressed sampling operator. To solve \eqref{hyperspec_form}, \citet{golbabaee12} implemented a proximal gradient method, which required solving a TV proximal operator for every spectral slice of the data volume in addition to solving the proximal operator of the nuclear norm (singular value thresholding) at every iteration of the algorithm \cite{combettes11}. For the large data volumes typically encountered in HSI, this can require significant computation per iteration.

Here we demonstrate the use of our structured matrix factorization method to perform hyperspectral compressed recovery by optimizing \eqref{app_obj}, where $\mathcal{A}(\cdot)$ is a compressive sampling function that applies a random-phase spatial convolution at each wavelength \cite{romberg09, golbabaee12}, $U$ contains estimated spectral features, and $V$ contains estimated spatial abundance features.\footnote{For HSI experiments, we set $\nu_u=\nu_{v_1}=0$ in \eqref{U_norm} and \eqref{V_norm}.}  Compressed recovery experiments were performed on the dataset from \citet{golbabaee12}\footnote{The data used are a subset of the publicly available AVARIS Moffet Field dataset.  We made an effort to match the specific spatial area and spectral bands of the data for our experiments to that used in \cite{golbabaee12} but note that slightly different data may have been used in our study.} at various subsampling ratios and with different levels of sampling noise. We limited the number of columns of $U$ and $V$ to 15 (the dataset has 256 $\times$ 256 pixels and 180 spectral bands), initialized one randomly selected pixel per column of $V$ to one and all others to zero, and initialized $U$ as $U=0$.

Figure \ref{hyperspect_examples} in Appendix \ref{sec:hyperspec_imgs} shows examples of the recovered images at one wavelength (spectral band $i=50$) for various subsampling ratios and sampling noise levels and Table \ref{hyperspect_table} shows the reconstruction recovery rates $\left\|X_{true} \!-\! UV^T\right\|_F / \left\|X_{true} \right\|_F$, where $X_{true}$ denotes the true hyperspectral image volume.  We note that even though we optimized over a highly reduced set of variables $([256 \times 256 \times 15+180 \times 15] / [256 \times 256 \times 180] \approx 8.4 \% )$  with very trivial initializations, we were able to achieve reconstruction error rates equivalent to or better than those in \citet{golbabaee12}.\footnote{The entries for NucTV in Table \ref{hyperspect_table} were adapted from \citep[Fig. 1]{golbabaee12}} Additionally, by solving the reconstruction in a factorized form, our method offers the potential to perform blind hyperspectral unmixing directly from the compressed samples without ever needing to reconstruct the full dataset, an application extension we leave for future work.
\begin{table}
\caption{Hyperspectral imaging compressed recovery error rates.}
\label{hyperspect_table}
\centering
\small
\begin{tabular}{c | ccc | ccc}
\hline
 & \multicolumn{3}{c |}{Our Method} & \multicolumn{3}{c}{NucTV} \\ 
\hline
Sample & \multicolumn{3}{c |}{Sampling SNR (dB)} & \multicolumn{3}{c}{Sampling SNR (dB)} \\
 Ratio & $\infty$ & 40 & 20 & $\infty$ & 40 & 20 \\ 
\hline
4:1   & 0.0209 & 0.0206 & 0.0565 & 0.01 & 0.02 & 0.06\\
8:1 	& 0.0223 & 0.0226 & 0.0589 & 0.03 & 0.04 & 0.08\\
16:1  & 0.0268 & 0.0271 & 0.0663 & 0.09 & 0.09 & 0.13\\
32:1  & 0.0393 & 0.0453 & 0.0743 & 0.21 & 0.21 & 0.24\\
64:1  & 0.0657 & 0.0669 & 0.1010 &  &  & \\
128:1 & 0.1140 & 0.1186 & 0.1400 &  &  & \\
\hline
\end{tabular}
\vspace{-3mm}
\end{table}

%% file: PAMI16_StructMatFact_conclusion.tex
\section{Conclusions}
We have proposed a highly flexible approach to structured matrix factorization, which allows specific structure to be promoted directly on the factors.  While our proposed formulation is not jointly convex in all of the variables, we have shown that under certain criteria a local minimum of the factorization is sufficient to find a global minimum of the product, offering the potential to solve the factorization using a highly reduced set of variables.

%% file: PAMI16_StructMatFact_supplement.tex
\section{Upper Bounding the Polar}
In many cases it is possible to derive semidefinite relaxations of the polar problem that upper-bound the polar solution.  Specifically, note that \eqref{eq:mat_polar} is equivalently reformulated as
\begin{equation*}
	\Omega_{\theta}^{\circ}(Z) = \sup_{u,v} \tfrac{1}{2}\left< \begin{bmatrix} 0 & Z \\ Z^T & 0 \end{bmatrix}, \begin{bmatrix} uu^T & uv^T \\ vu^T & vv^T \end{bmatrix} \right> \ST \theta(u,v) \leq 1.
\end{equation*}
If we make the change of variables $M = [u; v] [u; v]^T$, the problem is equivalent to optimizing over rank-1 semidefinite matrices $M$, provided there exists an equivalent function $\theta'(M)$ to enforce the constraint $\theta(u,v) \leq 1$ if $M$ is a rank-1 matrix.  For example, consider the case $\theta(u,v) = \tfrac{1}{2}(\|u\|_2^2 + \|v\|_2^2)$. Since $\theta(u,v) = \tfrac{1}{2}(\Tr(uu^T) + \Tr(vv^T)) = \tfrac{1}{2}\Tr(M)$, we have the following equivalent problems
\begin{align}
	& \Omega_{\theta}^{\circ} (Z) = \max_{u,v} \{ u^T Z v \st \tfrac{1}{2}(\Tr(uu^T) + \Tr(vv^T)) \leq 1 \} \! \\
	&=  \max_{M \succeq 0} \Big\{ \tfrac{1}{2} \langle \begin{bmatrix} 0 ~~~~~ Z \\ Z^T ~~ 0 \end{bmatrix},M \rangle \st \rank(M)=1,  \tfrac{1}{2}\Tr(M) \leq 1\Big\}. \nonumber
\end{align}
By removing the $\textnormal{rank}(M)=1$ constraint we then have a convex optimization problem on the space of positive semidefinite matrices that upper-bounds the polar problem,
\begin{equation}
	\Omega_{\theta}^{\circ}(Z) \leq \max_{M\succeq 0} \tfrac{1}{2} \Big\langle \begin{bmatrix} 0 & Z \\ Z^T & 0 \end{bmatrix},M \Big\rangle \st \tfrac{1}{2}\Tr(M) \leq 1,
\end{equation}
and if the solution to the above problem results in a rank-1 solution matrix $M$, which in this special case of $\theta(u,v)$ can be shown to be true via the S-procedure \cite{Boyd:2004}, then the inequality becomes an equality and the desired $(u,v)$ factors can be recovered from the largest singular vector of $M$.

This same idea can be extended to more general $\theta(u,v)$ regularization functions \cite{bach13} and has been used in techniques such as sparse PCA \cite{dAspremont:SIAM2007}. A few example functions on vectors $x$ and their equivalent function on $xx^T$ are provided in Table \ref{table:SDP_relax}, and these equivalences can be used to derive $\theta'(M)$ functions from a given $\theta(u,v)$ function.

\begin{table}[h]
\caption{Example equivalent forms of polar problem regularizers.}
\label{table:SDP_relax}
\centering
\small
\begin{tabular}{c | c}
\hline
\multicolumn{2}{c}{$h(x)\ \ \ \ \ = \ \ \ \ \ \ H(xx^T)$} \\
\hline
$\|x\|_F^2$ & $\Tr(xx^T)$ \\
$\|x\|_1^2$ & $\|xx^T\|_1$ \\
$\|x\|_{\infty}^2$ & $\|xx^T\|_{\infty}$ \\
$\|Ax\|_1^2$ & $\|Axx^T A^T \|_1$ \\
$\|Ax\|_1\|x\|_F$ & $\sum_{i} \|(xx^T A)_i \|_F$ \\
$\delta_{\Re_+}(x)$ & $\delta_{\Re_+}(xx^T)$ \\
\hline
\end{tabular}
\end{table}

While, unfortunately, in general there is no guarantee that the solution to the semidefinite relaxation will be a rank-1 $M$ matrix for an arbitrary $\theta(u,v)$ regularization function, for some $\theta(u,v)$ one can prove bounds about how close the upper-bound of the polar obtained from the semidefinite relaxation will be to the true value of the polar \cite{bach13}.

\section{Solving the L1-TV Proximal Operator}
\label{sec:L1TV_prox}

Recall that for the applications we study here it is necessary to solve proximal operators for norms of the form given in \eqref{U_norm} and \eqref{V_norm}, optionally subject to non-negativity constraints.  From Theorem \ref{prop:l2prox} we have an efficient means to solve the proximal operator of \eqref{U_norm} and \eqref{V_norm} optionally subject to non-negativity constraints, as we can first solve a proximal operator of the form
\begin{equation}
	\label{eq:L1TV_prox}
	\argmin_x \tfrac{1}{2} \|y-x\|_F^2 + \nu_1 \|x\| + \nu_{TV} \|x\|_{TV} \ \textnormal{(optionally s.t.)} \ x \geq 0
\end{equation}
and then calculate the proximal operator of the $l_2$ norm with the solution to \eqref{eq:L1TV_prox} as the argument.  The only component that is missing to apply Algorithm \ref{algorithm_low_rank_fact} to the proposed class of regularization functions is how to solve the proximal operator of the $l_1$ norm plus the total-variation pseudo-norm (and optionally subject to non-negativity constraints).  To address this issue, note that \eqref{eq:L1TV_prox} is equivalent to solving the problem
\begin{align}
	\label{eq:gen_lasso}
	\argmin_x & \tfrac{1}{2} \|y-x\|_F^2 + \|G x\|_1 \ \textnormal{(optionally s.t.)} \ x \geq 0\\
	G &= \left[ \begin{array}{c} \nu_1 I \\ \nu_{TV} \Delta \end{array} \right],
\end{align}
where $\Delta$ is a matrix that takes the difference between neighboring elements of $x$.  Using standard Lagrangian duality arguments, such as those presented in \cite{Tibshirani:AnnStat2011}, it is easily shown that the dual problem of \eqref{eq:gen_lasso} is equivalent to
\begin{align}
	\label{eq:dual_L1TV}
	\min_{\gamma} \tfrac{1}{2} & \|y-G^T \gamma\|_F^2 \st \|\gamma\|_{\infty} \leq 1 \\
	\nonumber & \textnormal{(optionally s.t.)} \ y-G^T \gamma \geq 0
\end{align}
with the primal-dual relationship $x = y-G^T \gamma$.  To solve \eqref{eq:dual_L1TV}, we note that due to the special structure of $G$, one can calculate the global optimum of an individual element of $\gamma$ extremely quickly if the other elements in $\gamma$ are held constant.  Thus, to solve \eqref{eq:dual_L1TV} we cycle through making updates to the elements of $\gamma$ while checking the duality gap for convergence.  For the values of the $\nu$ parameters typically used in our experiments, this strategy converges after a relatively small number of cycles through the $\gamma$ variables, and due to the fact that the updates to the $\gamma$ variables themselves are very easy to calculate this strategy provides a very efficient means of solving the proximal operator in \eqref{eq:L1TV_prox}.

\section{Supplementary Proofs}
\label{supplementary}

\subsection{Proof of Proposition \ref{prop:omega_prop}}
\label{sec:omega_prop_proof}

\begin{proof}
For brevity of notation, we will notate the optimization problem in \eqref{eq:omega_mat_def} as
\begin{equation}
	\Omega_{\theta} (X) \equiv \inf_{U,V:UV^T=X} \sum_{i=1}^r \theta(U_i,V_i),
\end{equation}
where recall that $r$ is variable although it is not explicitly notated.

\begin{enumerate}
\item By the definition of $\Omega_\theta$ and the fact that $\theta$ is always non-negative, we always have $\Omega_{\theta}(X) \geq 0 \ \ \forall X$.  Trivially, $\Omega_{\theta}(0) = 0$ since we can always take $(U,V) = (0,0)$ to achieve the infimum.  For $X \neq 0$, note that $\sum_{i=1}^r \theta(U_i,V_i) > 0$ for any $(U,V) \ \ \textnormal{s.t.} \ \ UV^T = X$ and $r$ finite.  Property 5 shows that the infimum can be achieved with $r$ finite, completing the result.

\item The result is easily seen from the positive homogeneity of $\theta$,
\begin{align}
	\Omega_{\theta}(\alpha X) = &\inf_{U,V:UV^T=\alpha X} \sum_{i=1}^r \theta(U_i,V_i) =  \\
	&\inf_{U,V:(\alpha^{-1/2} U)(\alpha^{-1/2} V)^T = X} \sum_{i=1}^r \theta(U_i,V_i) = \nonumber\\
	&\inf_{\bar U, \bar V: \bar U \bar V^T=X} \alpha \sum_{i=1}^r \theta(\bar U_i, \bar V_i) = \alpha \Omega_{\theta}(X), \nonumber
\end{align}
where the equality between the middle and final lines is simply due to the change of variables $(\bar U,\bar V) = (\alpha^{-1/2} U,\alpha^{-1/2} V)$ and the fact that $\theta$ is positively homogeneous with degree 2.

\item If either $\Omega_{\theta}(X) = \infty$ or $\Omega_{\theta}(Z)=\infty$ then the inequality is trivially satisfied.  Considering any $(X,Z)$ pair such that $\Omega_{\theta}$ is finite for both $X$ and $Z$, for any $\epsilon > 0$ let $UV^T=X$ be an $\epsilon$ optimal factorization of $X$.  Specifically, $\sum_{i=1}^{r_x} U_i V_i^T= X$ and $\sum_{i=1}^{r_x} \theta(U_i,V_i) \leq \Omega_{\theta}(X) + \epsilon$.  Similarly, let $\bar U \bar V^T$ be an $\epsilon$ optimal factorization of $Z$; $\sum_{i=1}^{r_z} \bar U_i \bar V_i^T =Z$ and $\sum_{i=1}^{r_z} \theta(\bar U_i, \bar V_i) \leq \Omega_{\theta}(Z) + \epsilon$.  Taking the horizontal concatenation $([U \ \bar U],[V \ \bar V])$ gives $[U \ \bar U][V \ \bar V]^T = X+Z$, and $\Omega_{\theta}(X+Z) \leq \sum_{i=1}^{r_x} \theta(U_i,V_i) + \sum_{j=1}^{r_z} \theta(\bar U_j,\bar V_j) \leq \Omega_{\theta}(X) + \Omega_{\theta}(Z) + 2\epsilon$.  Letting $\epsilon$ tend to 0 completes the result.

\item Convexity is given by the combination of properties 2 and 3.

\item Let $\Gamma \subset \Re^{DN}$ be defined as a subset of rank-1 matrices
\begin{equation}
\Gamma = \{ X : \exists uv^T, \ uv^T=X, \ \theta(u,v) \leq 1 \}.
\end{equation}
Note that because of property 3 in Definition \ref{def:rank1_meas}, for any non-zero $X \in \Gamma$ there exists $\mu \in [1,\infty)$ such that $\mu X$ is on the boundary of $\Gamma$, so $\Gamma$ and its convex hull are compact sets. 

Further, note that $\Gamma$ contains the origin by definition, so as a result, we can define $\sigma_{\Gamma}$ to be a gauge function on the convex hull of $\Gamma$, 
\begin{equation}
\sigma_{\Gamma}(X) = \inf_{\mu} \{\mu : \mu \geq 0, \ X \in \mu \ \textnormal{conv} (\Gamma) \}.
\end{equation}
Since the infimum w.r.t. $\mu$ is linear and constrained to a compact set, it must be achieved.  Therefore, there must exist $\mu_{opt} \geq 0$, $\{ \beta \in \Re^{DN} : \beta_i \geq 0 \ \forall i, \ \sum_{i=1}^{DN} \beta_i = 1 \}$, and $\{(\bar U,\bar V) : (\bar U_i, \bar V_i) \in \Gamma \ \forall i\in [1,DN]\}$ such that $X = \mu_{opt} \sum_{i=1}^{DN} \beta_i \bar U_i \bar V_i^T$ and $\sigma_{\Gamma}(X) = \mu_{opt}$.

Combined with positive homogeneity, this gives that $\sigma_{\Gamma}$ can be defined identically to $\Omega_{\theta}$, but with the additional constraint $r \leq DN$,  
\begin{equation}
\label{eq:sigma_gamma_equiv}
\sigma_{\Gamma}(X) \equiv \inf_{r \in [1,DN]} \inf_{\substack{U \in \Re^{D \times r} \\ V \in \Re^{N \times r} \\ UV^T=X}} \sum_{i=1}^r \theta(U_i,V_i).
\end{equation}
This is seen by noting that we can take $(U_i, V_i) = ((\mu_{opt} \beta_i)^{1/2} \bar U_i, (\mu_{opt} \beta_i)^{1/2} \bar V_i)$ to give
\begin{equation}
\begin{split}
\mu_{opt}  = \sigma_{\Gamma}(X) &\leq \sum_{i=1}^{DN} \theta(U_i,V_i) = \mu_{opt} \sum_{i=1}^{DN} \beta_i \theta(\bar U_i,\bar V_i)  \\
&\leq \mu_{opt} \sum_{i=1}^{DN} \beta_i = \mu_{opt},
\end{split}
\end{equation}
and shows that a factorization of size $r \leq \card(X)$ which achieves the infimum $\mu_{opt} = \sigma_{\Gamma}(X)$ must exist.  Clearly from \eqref{eq:sigma_gamma_equiv} $\sigma_{\Gamma}$ is very similar to $\Omega_{\theta}$.  To see that the two functions are, in fact, the same function, we use the one-to-one correspondence between convex functions and their Fenchel duals.  In particular, in the proof of \ref{prop:omega_subgrad}) we derive the Fenchel dual of $\Omega_{\theta}$ and using an identical series of arguments one can derive an identical Fenchel dual for $\sigma_{\Gamma}$.  The result is then completed as the one-to-one correspondence between convex functions and their duals gives that $\sigma_{\Gamma}(X) = \Omega_{\theta}(X)$. 

\item Note that from properties 1-3 we have established all of the requirements for a norm, except for invariance w.r.t. negative scaling, i.e., we must show that $\Omega_{\theta}(-X) = \Omega_{\theta}(X)$.  To see that this will be true, consider the case where $\theta(u,v) = \theta(-u,v)$.  This gives, 
\begin{equation}
	\begin{split}
	\Omega_{\theta}(-X) &= \inf_{U,V:UV^T = -X} \sum_{i=1}^r \theta(U_i,V_i) = \\
	&\inf_{U,V:(-U)V^T=X} \sum_{i=1}^r \theta(U_i,V_i) = \\
	&\inf_{U,V: \bar U V^T=X} \sum_{i=1}^r \theta(\bar U_i,V_i) = \Omega_{\theta}(X),
	\end{split}
\end{equation}
where the last equality comes from the change of variables $\bar U = -U$ and the fact that $\theta(u,v) = \theta(-u,v)$.  An identical series of arguments shows the result for the case when $\theta(u,v) = \theta(u,-v)$.

\item Because $\Omega_{\theta}$ is a convex function, we have from Fenchel duality that $W \in \partial \Omega_\theta(X)$ if and only if $\left<W,X\right> = \Omega_\theta(X) + \Omega_\theta^*(W)$ where $\Omega_\theta^*$ denotes the Fenchel dual of $\Omega_\theta$.  To derive the Fenchel dual, recall the definition $\Omega_\theta^*(W) \equiv \sup_Z \left<W,Z\right> - \Omega_\theta(Z)$, giving the equivalent problem
\begin{align}
	\label{eq:fenchel_sum}
	\Omega_\theta^*(W) &= \sup_{r \in \Nplus} \sup_{ \substack{U \in \Re^{D \times r} \\ V \in \Re^{N \times r}} } \left<W,UV^T\right> - \sum_{i=1}^r \theta(U_i,V_i) \nonumber \\ 
	&= \sup_{r \in \Nplus} \sup_{ \substack{U \in \Re^{D \times r} \\ V \in \Re^{N \times r}}} \sum_{i=1}^r \left( U_i^T W V_i^T - \theta(U_i,V_i) \right)
\end{align}
Note that if there exists $(u,v)$ such that $u^T W v > \theta(u,v)$, then $\Omega^*_\theta(W) = \infty$ since \eqref{eq:fenchel_sum} grows unbounded by taking $(\alpha u, \alpha v)$ as $\alpha \rightarrow \infty$.  Conversely, if $u^T W v \leq \theta(u,v)$ for all $(u,v)$ then all of the terms in the summation in \eqref{eq:fenchel_sum} will be non-positive, so the superemum is achieved by taking $(U,V) = (0,0)$.  As a result we have,
\begin{equation}
	\label{eq:omega_fenchel}
	\Omega_\theta^*(W) = \left\{ \begin{array}{cc} 0 & u^T W v \leq \theta(u,v) \ \ \forall(u,v) \\ \infty & \textnormal{otherwise}. \end{array} \right.
\end{equation}
Combining this result with the characterization of the subgradient from Fenchel duality above completes the result.

\item First, we show that $W \in \partial \Omega_\theta(X)$.  By contradiction, assume $W \notin \partial \Omega_\theta(X)$.  From the fact that $u^T W v \leq \theta(u,v)$ $\forall (u,v)$ and \eqref{eq:omega_fenchel} we have that $\Omega_\theta^*(W) = 0$.  Further, we have $\sum_{i=1}^r \theta(U_i,V_i) = \sum_{i=1}^r U_i^T W V_i = \left<W, X\right> < \Omega_\theta(X) + \Omega_\theta^*(W) = \Omega_\theta(X)$, where the strict inequality comes from Fenchel duality and the assumption that $W \notin \partial \Omega_\theta(X)$.  However, because $UV^T=X$ is a factorization of $X$ this violates the definition of $\Omega_\theta$, producing the contradiction.  To see that $UV^T$ is an optimal factorization of $X$, we again proceed by contradiction and assume that $UV^T$ is not an optimal factorization.  This gives $\Omega_\theta(X) < \sum_{i=1}^r \theta(U_i,V_i) = \sum_{i=1}^r U_i^T W V_i = \left< W,X \right> = \Omega_\theta(X) + \Omega^*_\theta(W) = \Omega_\theta(X)$, producing the contradiction.
\end{enumerate}
\end{proof}

\subsection{Proof of Proposition \ref{prop:1st_order}}
\label{sec:1st_order_proof}
\begin{proof}
Note that condition \ref{cond:Q} is trivially satisfied, as this is simply the first-order optimality requirement w.r.t. $Q$.  Thus, we are left with showing that condition \ref{cond:UV1} is also satisfied.  Let $\theta_p (u,v) = \sigma_u(u) \sigma_v(v)$ and $\theta_s(u,v) = \tfrac{1}{2} (\sigma_u(u)^2 + \sigma_v(v)^2)$.  Note that the following are easily shown from basic properties of subgradients of gauge functions
	\begin{equation}
		\label{eq:gauge_subgrad}
		\begin{split}
		&\left<u, \partial_u \theta_p (u,v) \right> = \left<v, \partial_v \theta_p (u,v) \right> = \theta_p(u,v) \\
		&\left<u, \partial_u \theta_s (u,v) \right> = \sigma_u(u)^2 \\
		&\left<v, \partial_v \theta_s (u,v) \right> = \sigma_v(v)^2.
		\end{split}
	\end{equation}
	The first-order optimality conditions w.r.t. $U_i$ and $V_i$ are
	\begin{align}
		0 &\in \nabla_X \ell (Y,\tilde{U}\tilde{V}^T,\tilde{Q}) \tilde{V}_i + \lambda \partial_u \theta(\tilde{U}_i,\tilde{V}_i) \\
		0 &\in \nabla_X \ell (Y,\tilde{U}\tilde{V}^T,\tilde{Q})^T \tilde{U}_i + \lambda \partial_v \theta(\tilde{U}_i,\tilde{V}_i).
	\end{align}
	Left multiplying the two above inclusions by $\tilde{U}_i^T$ and $\tilde{V}_i^T$, respectively gives
	\begin{align}
		0 &\in \tilde{U}_i^T \nabla_X \ell  (Y,\tilde{U}\tilde{V}^T,\tilde{Q}) \tilde{V}_i + \lambda \left< \tilde{U}_i, \partial_u \theta(\tilde{U}_i,\tilde{V}_i) \right>\\
		0 &\in \tilde{V}_i^T \nabla_X \ell (Y,\tilde{U}\tilde{V}^T,\tilde{Q})^T \tilde{U}_i + \lambda \left< \tilde{V}_i, \partial_v \theta(\tilde{U}_i,\tilde{V}_i) \right>.
	\end{align}
	Since this is true for all $(\tilde{U}_i,\tilde{V}_i)$ pairs, substituting \eqref{eq:gauge_subgrad} and rearranging terms  shows that condition \ref{cond:UV1} of Corollary \ref{cor:global_min_conds} is satisfied for both $\theta_s$ and $\theta_p$, completing the result. 
\end{proof}

\subsection{Proof of Proposition \ref{prop:approx_bound}}
\label{sec:approx_bound_proof}
\begin{proof} Let $\tilde{X} = \tilde{U}\tilde{V}^T$.  From the (strong) convexity of $\ell$, we have that, for any $W \in \partial_Q \ell(Y,\tilde{X},\tilde{Q})$,
	\begin{align*}
		\ell(  Y,   \hat{X},\hat{Q}) 
		& \geq  \ell(Y,\tilde{X},\tilde{Q}) + \tfrac{m_X}{2} \|\tilde{X} - \hat{X}\|_F^2 + \tfrac{m_Q}{2} \|\tilde{Q} - \hat{Q}\|_F^2 \\
		& + \left<\nabla_X \ell(Y,\tilde{X},\tilde{Q}),\hat{X}-\tilde{X} \right> + \big\langle W,\hat{Q}-\tilde{Q} \big \rangle.
	\end{align*}
	From condition \ref{cond:Q} of Corollary \ref{cor:global_min_conds} we can take 
	$W \!=\! 0$, and from condition \ref{cond:UV1} we have $\big\langle-\nabla_X \ell(Y,\tilde{X},\tilde{Q}),\tilde{X} \big\rangle = \lambda \sum_i \theta(\tilde{U}_i,\tilde{V}_i)$. Applying these facts and rearranging terms gives 
	\begin{equation}
		\label{eq:polar_bound}
		\begin{split}
		\ell(Y,\tilde{X},\tilde{Q}) 
		&- \ell(Y,\hat{X},\hat{Q}) + \lambda \sum_i \theta(\tilde{U}_i,\tilde{V}_i) \\
		&\leq \lambda \left< \tfrac{-1}{\lambda} \nabla_X \ell(Y,\tilde{X},\tilde{Q}),\hat{X} \right> \\
		&-\tfrac{m_X}{2}\|\tilde{X}-\hat{X}\|_F^2 - \tfrac{m_Q}{2} \| \tilde{Q}-\hat{Q} \|_F^2.
		\end{split}
	\end{equation}
	Recall that from polar duality we also have $\forall (X,Z)$, $\left<X,Z\right> \leq \Omega_{\theta}(X) \Omega_{\theta}^{\circ}(Z)$, which implies 
	\begin{equation*}
		\left< \tfrac{-1}{\lambda} \nabla_X \ell(Y,\tilde{X},\tilde{Q}),\hat{X} \right> \leq \Omega_{\theta}(\hat{X}) \Omega_{\theta}^{\circ}(\tfrac{-1}{\lambda} \nabla_X \ell(Y,\tilde{X},\tilde{Q})).
	\end{equation*}
	Substituting this into \eqref{eq:polar_bound} we obtain
	\begin{align*}
		f(\tilde U, \tilde V, \tilde Q) - \ell(Y,\hat{X},\hat{Q})
		&\leq \lambda \Omega_{\theta}(\hat{X}) \Omega_{\theta}^{\circ}(\tfrac{-1}{\lambda} \nabla_X \ell(Y,\tilde{X},\tilde{Q})) \\
		&-\tfrac{m_X}{2}\|\tilde{X}-\hat{X}\|_F^2 - \tfrac{m_Q}{2} \| \tilde{Q}-\hat{Q} \|_F^2.
	\end{align*}
	Subtracting $\lambda \Omega_{\theta}(\hat{X})$ from both sides of the above inequality completes the result.
\end{proof}

\subsection{Proof of Theorem \ref{thm:non_inc_path}}
\label{sec:non_inc_path_proof}
\begin{proof} Clearly if $(U_{init},V_{init},Q_{init})$ is not a local minimum, then we can follow a decreasing path until we reach a local minimum.  Having arrived at a local minimum, $(\tilde{U},\tilde{V},\tilde{Q})$, if $(\tilde{U}_i,\tilde{V}_i)=(0,0)$ for any $i \in [r]$ then from Theorem \ref{thm:mat_fact_0_min} we must be at a global minimum.  Similarly, if for any $i_0 \in [r]$ we have $\tilde{U}_{i_0}\tilde{V}_{i_0}^T = 0$ then we can scale the columns $(\alpha \tilde{U}_{i_0},\alpha \tilde{V}{i_0})$ as $\alpha$ goes from $1 \rightarrow 0$ without increasing the objective function.  Once $\alpha=0$ we will then have an all zero column in the factorization, so from Theorem \ref{thm:mat_fact_0_min} we are either at a global minimum or a local descent direction must exist from that point.  We are thus left to show that the Algorithm continues to decrease the objective from any local minima such that $\tilde U_i,\tilde{V}_i^T \neq 0$ for all $i \in [r]$.

We first note that appending an all-zero column to $(U,V)$ does not change the objective function, and once the all-zero column has been appended we are again either at a global minimum from Theorem \ref{thm:mat_fact_0_min} or a local descent direction must exist.  As a result, the only part left to show is that the step of scaling the factors by $(1+\beta)^{1/2}$ does not increase the objective function and that $r$ will never increase beyond $\max \{ DN+1,r_{init} \}$.

To see this, consider the case where a non-zero vector, $\beta \in \Re^r / 0$, exists such that $\sum_{i=1}^r \beta_i \tilde U_i \tilde V_i^T = 0$.  Without loss of generality, assume $\beta$ is scaled so that the minimum value in $\beta$ is equal to $-1$ and let $U_{1\pm\epsilon \beta} = [(1\pm\epsilon \beta_1)^{1/2}\tilde U_1, \ldots, (1\pm\epsilon \beta_r)^{1/2} \tilde U_r]$ and $V_{1\pm\epsilon \beta} = [(1\pm\epsilon \beta_1)^{1/2}\tilde V_1, \ldots, (1\pm\epsilon \beta_r)^{1/2} \tilde V_r]$ for some $\epsilon > 0$.  From this, because we are at a local minimum, then similar to the proof of Theorem \ref{thm:mat_fact_0_min} for some $\delta>0$ we have $\forall \epsilon \in (0, \delta)$ that 
\begin{equation}
	\label{eq:beta_search}
	\begin{split}
	\ell(Y,& U_{1 \pm \epsilon \beta} V_{1 \pm \epsilon \beta}^T, \tilde Q)\! + \! \lambda \sum_{i=1}^r\theta((1 \pm \epsilon \beta_i)^{1/2} \tilde U_i,(1 \pm \epsilon \beta_i)^{1/2} \tilde V_i) \\
	&= \ell(Y, \tilde U \tilde V^T,\tilde Q) + \lambda \sum_{i=1}^r (1 \pm \epsilon \beta_i) \theta(\tilde U_i,\tilde V_i) \\
	& \geq \ell(Y,\tilde U \tilde V^T,\tilde Q) + \lambda \sum_{i=1}^r \theta (\tilde U_i,\tilde V_i),
	\end{split}
\end{equation}
where the equality is due to the construction of $\beta$ and $(U_{1 \pm \epsilon \beta},V_{1 \pm \epsilon \beta})$ along with the positive homogeneity of $\theta$.  Rearranging terms gives
\begin{align}
	0 \leq \lambda \epsilon \sum_{i=1}^r \beta_i \theta(\tilde U_i, \tilde V_i) \leq 0 \\
	\label{eq:0_theta_sum}
	\implies \sum_{i=1}^r \beta_i \theta(\tilde U_i,\tilde V_i) = 0.
\end{align}

From this result, note that by using the same expansion as in \eqref{eq:beta_search}, \eqref{eq:0_theta_sum} gives that the value of $f(U_{1+\epsilon \beta},V_{1+\epsilon \beta},\tilde {Q})$ is unchanged as $\epsilon$ goes from $0 \rightarrow 1$, and by construction, $(U_{1+\beta},V_{1+\beta})$ contains an all-zero column. Thus, $(U_{1+\beta},V_{1+\beta},\tilde{Q})$ is either a global minimum due to Theorem \ref{thm:mat_fact_0_min} or a local descent direction must exist.  The result is completed that noting that a non-zero $\beta$ vector is guaranteed to always exists whenever $r > DN$.
\end{proof}

\subsection{Proof of Proposition \ref{prop:l2prox}}
\label{sec:l2prox_proof}
\begin{proof}
Note that we can equivalently solve the proximal operator by introducing another variable subject to an equality constraint,
\begin{equation}
\prox_{\psi}(y) = \argmin_{x,z : x=z } \frac{1}{2} \|y-x\|_2^2+\lambda \sigma_C(z) + \lambda_2 \| x \|_2.
\end{equation}
This gives the Lagrangian
\begin{equation}
L(x,z,\gamma) = \frac{1}{2}\|y-x\|_2^2+ \lambda \sigma_C( z ) + \lambda_2 \| x \|_2 + \left<\gamma,x-z\right>.
\end{equation}
Minimizing the Lagrangian w.r.t. $z$, we obtain the negative Fenchel dual of $\lambda \sigma_C$, which is an indicator function on the polar set
\begin{equation}
\begin{split}
\min_z & \lambda \sigma_C(z) - \left<\gamma, z \right> = -(\lambda \sigma_C(\gamma))^* \\ & = 
\left\{\begin{array}{cc} 0 & \sigma_C^{\circ}(\gamma) \leq \lambda \\ -\infty & \textnormal{otherwise.} \end{array} \right.
\end{split}
\end{equation}
Minimizing the Lagrangian w.r.t. $x$, we obtain
\begin{align}
&\min_x \frac{1}{2}\|y-x\|_2^2+\lambda_2 \| x \|_2 + \left<\gamma, x \right> = \label{prox_lagrang:1} \\
&\min_x \frac{1}{2}\|y-\gamma-x\|_2^2+\lambda_2 \| x \|_2 + \left<\gamma, y\right>-\frac{1}{2}\|\gamma\|_2^2 = \label{prox_lagrang:2}\\
&\left\{ \begin{array}{cc} \frac{1}{2} \| y \|_2^2 &  \| y-\gamma \|_2 \leq \lambda_2 \\
	\frac{1}{2} \| y \|_2^2 -\frac{1}{2}\left( \| y-\gamma \|_2-\lambda_2 \right)^2& \text{else} \end{array}
	\right. \label{prox_lagrang:3}
\end{align}
where the minimum value for $x$ is achieved at $x = \prox_{\lambda_2 \| \cdot \|_2}(y-\gamma)$.  The relation between 
\eqref{prox_lagrang:1} and \eqref{prox_lagrang:2} is easily seen by expanding the quadratic terms, while the relation
between \eqref{prox_lagrang:2} and \eqref{prox_lagrang:3} is given by the fact that \eqref{prox_lagrang:2} is the standard 
proximal operator for the $l_2$ norm plus terms that do not depend on $x$.  Plugging the solution of the proximal operator
of the $l_2$ norm (noting that the $l_2$ norm is self dual) into \eqref{prox_lagrang:2} gives \eqref{prox_lagrang:3}.  The dual of the original problem thus becomes maximizing \eqref{prox_lagrang:3}
w.r.t. $\gamma$ subject to $\sigma_C^{\circ}(\gamma)  \leq \lambda$.
We note that \eqref{prox_lagrang:3} is monotonically non-decreasing as $\| y-\gamma \|_2$ decreases, so the dual
problem is equivalent to minimizing $\| y-\gamma \|_2$ (or equivalently $\|y - \gamma \|_2^2$) subject to $\sigma^{\circ}_C(\gamma) \leq \lambda$.  Combining these results with the primal-dual relation 
$x=\prox_{\lambda_2 \| \cdot \|_2}(y-\gamma)$, we have
\begin{equation}
	\prox_{\psi}(y) = \prox_{\lambda_2 \|\cdot\|_2} (y-\gamma_{opt}),
\end{equation}
where $\gamma_{opt}$ is the solution to the optimization problem
\begin{equation}
	\gamma_{opt} = \argmin_{\gamma} \|y-\gamma\|_2^2 \st \sigma^{\circ}_C(\gamma) \leq \lambda.
\end{equation}
Recall that from the Moreau identity, the proximal operator of the $\lambda \sigma_C$ gauge is given by
\begin{equation}
	\prox_{\lambda \sigma_C} (y) = y - \argmin_{\gamma} \|y-\gamma\|_2^2 \st \sigma_C^{\circ}(\gamma) \leq \lambda,
\end{equation}
which completes the result, as the above equation implies $\prox_{\lambda \sigma_C}(y) = y-\gamma_{opt}$.
\end{proof}

\onecolumn
\newpage
\section{Hyperspectral Reconstruction Results}
\label{sec:hyperspec_imgs}
\begin{figure*}[htb]
\centering
\includegraphics[width=\textwidth]{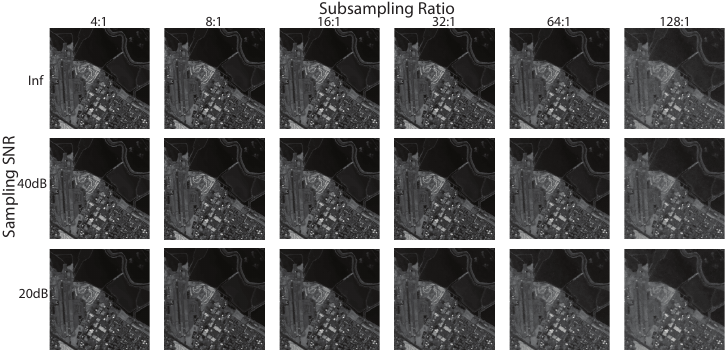}\\
\caption[Hyperspectral compressed recovery results.] {Hyperspectral compressed recovery results. Example reconstructions from a single spectral band ($i=50$) under different subsampling
ratios and sampling noise levels. Compare with \citet[Fig. 2]{golbabaee12}.}
\label{hyperspect_examples}
\end{figure*} 